\documentclass{article}
\usepackage[margin=2.5cm]{geometry}

\usepackage{algorithmic}
\usepackage{algorithm}
\usepackage{hyperref}

\usepackage{latexsym}
\usepackage{graphicx}
\DeclareGraphicsExtensions{.pdf,.png,.jpg,.jpeg }
\usepackage{bm}
\usepackage{amssymb,amsfonts,amsmath,amsthm}

\newtheorem{thm}{Theorem}[section]
\newtheorem{lem}[thm]{Lemma}

\newtheorem{prop}[thm]{Proposition}

\newtheorem{defn}{Definition}[section]
\newtheorem{rem}{Remark}[section]
\newtheorem{ex}{Example}[section]

\begin{document}
\title{A Further Generalization of the Finite-Population Geiringer-like Theorem for POMDPs to Allow Recombination Over
Arbitrary Set Covers}

\author{Boris Mitavskiy\thanks{\href{mailto:bom4@aber.ac.uk}{Email: bom4@aber.ac.uk}}~  and Jun He\\Department of Computer Science\\ Aberystwyth University, Aberystwyth, SY23 3DB, U.K.}

\maketitle
\begin{abstract}
A popular current research trend deals with expanding the Monte-Carlo tree search sampling methodologies to the environments with uncertainty and incomplete information. Recently a finite population version of Geiringer theorem with nonhomologous recombination has been adopted to the setting of Monte-Carlo tree search to cope with randomness and incomplete information by exploiting the entrinsic similarities within the state space of the problem. The only limitation of the new theorem is that the similarity relation was assumed to be an equivalence relation on the set of states. In the current paper we lift this ``curtain of limitation" by allowing the similarity relation to be modeled in terms of an arbitrary set cover of the set of state-action pairs.
\end{abstract}

\section{Introduction}\label{IntroSect}
In recent years Monte-Carlo sampling methods, such as Monte Carlo tree search, have achieved tremendous success in model free reinforcement learning with, perhaps the most
celebrated example, being the computer Go software that has beaten the top human player (see \cite{MCTGoChesl}, for instance). Unlike the traditional techniques such as the mini-max method (that, by the way, have not succeeded for ``Go"), Monte-Carlo Tree search (MCT) is based on running simulated self-plays, called ``rollouts" until the end of the game when a terminal state with a known payoff has been encountered. A position in a game is represented by a state-action pair $\vec{s} = (s, \, \vec{\alpha})$ where $\vec{\alpha} = (\alpha_1, \, \alpha_2, \ldots, \alpha_{l(s)})$ is the collection of actions (or moves) available in a position $s$. The goal is to evaluate actions through back-propagation and averaging.  Much effort in the current research is devoted to widening the range of applicability of the method in
the environments with randomness and incomplete information. Recently, a finite population Geiringer Theorem with non-homologous recombination has been adopted to the setting of Monte-Carlo tree search, based on which efficient parallel algorithms that exploit the intrinsic similarities within the state-action space to increase exponentially the size of a simulated sample of rollouts can be developed. The importance of similarity relations when coping with POMDPs is emphasized by other researchers, see, for instance, \cite{POMDPsSymmetries}. The main idea is that the algorithms sample directly from a long term probability distribution of repeated recombination applications defined in terms of the similarities. The only limitation of the new theorem is that the notion of similarity is limited to equivalence relations partitioning the set of state-action pairs. The current work lifts this ``curtain of limitation" by allowing the similarity relation to be any set cover of the set of state-action pairs. The reader is strongly encouraged to familiarize themselves with the the first four sections of \cite{MitavRowGeirNewMain} prior to reading the current sequel paper, where a much more detailed introduction as well as mathematical support is provided.
\section{Mathematical Framework and Notation}\label{MathFrameworkSect}
\subsection{Set Cover of the State Space}\label{equivSimSect}
Let $S$ denote the set of states (enormous but finite in the current framework). Formally each state $\vec{s} \in S$ is an ordered pair $(s, \vec{\alpha})$ where $\vec{\alpha}$ is the set of actions an agent can possibly take when in the state $\vec{s}$. Let $\mathcal{C}$ denote an arbitrary set cover of $S$ (i.e. $\mathcal{C} \subseteq \mathcal{P}(S)$ is a collection of subsets of $S$ such that $\bigcup_{O \in \mathcal{C}}O = S$). Given a set $O \in \mathcal{C}$, for any two states $\vec{s}_1$ and $\vec{s}_2 \in O$ we will say that $\vec{s}_1$ and $\vec{s}_2$ are $O$-\emph{similar states} and write $\vec{s}_1 \overset{O}{\sim} \vec{s}_2$. Intuitively, the sets $O$ represent certain measure of similarity between the states $\vec{s}_1$ and $\vec{s}_2$. In practice, of course, if $\vec{s}_1 \overset{O}{\sim} \vec{s}_2$, the corresponding sets of actions $\vec{\alpha}_1$ and $\vec{\alpha}_2$ must be related in some kind of fashion: for instance, one may require that there are functions $f_{O, \, \vec{s}_1, \, \vec{s}_2}: \vec{\alpha}_1 \rightarrow \vec{\alpha}_2$ and $f_{O, \, \vec{s}_2, \, \vec{s}_1}: \vec{\alpha}_2 \rightarrow \vec{\alpha}_1$ that provide a natural similarity correspondence among the actions. Needless to say, the choice or design of such correspondences goes hand in hand with the choice or design of the set cover $\mathcal{C}$. In fact, there is a variety of ways in which this can be modeled depending on the specific applications. We leave the detailed investigations for future research.

Unlike the framework in \cite{MitavRowGeirNewMain}, the current setting allows various degrees and types of similarity relations that are not limited to partitions of $S$ induced by equivalence relations. Consider, for instance, a set covering induced by a distance function $d: S \times S \rightarrow [0, \, \infty)$ satisfying the usual axioms of a pseudo-metric: $\forall \, x, \, y, \, z \in S$ we have $d(x, \, x) = 0$, $d(x, \, y) = d(y, \, x)$ and $d(x, \, z) \leq d(x, \, y)+d(y, \, z)$. Such a pseudo-metric naturally induces a neighborhood structure on the set of states $S$: $\mathcal{C} = \{B(x, \, \epsilon) \, | \, x \in S \text{ and } \epsilon > 0\}$ where $B(x, \, \epsilon) = \{y \, | \, y \in S \text{ and }d(x, y) < \epsilon\}$ is a clopen ball\footnote{Since the set $S$ is finite, the pseudo-metric (i.e. the distance function) can take only finitely many values thereby inducing the discrete topology on the set $S$ (see, for instance, \cite{SimmonsG} to learn about basic point-set topology). The induced topological space is then totally disconnected so that every open ball is also a closed ball, usually abbreviated as ``clopen". In fact, in such a case it is sufficient to consider only integer-valued pseudo-metrics, yet depending on a specific application it may be sometimes more convenient to consider real or rational valued pseudo-metrics.}. Notice further that $\mathcal{C}$ has a number of subcovers such as, for instance, $\mathcal{C}_r = \{B(x, \, \epsilon) \, | \, x \in S \text{ and } r > \epsilon > 0\}$ where $r > 0$ and each one of these could be used in specific applications. As mentioned in the introduction, due to an overwhelming number of states as well as incomplete information and randomness, some sort of a coarse graining of the set of states is inevitable in applications. As pointed out in \cite{MitavRowGeirNewMain}, the primary idea of the Geiringer-like theorems for decision making in the environments with randomness and incomplete information is to exploit the similarity relations on the set of states to estimate the average action payoffs based on an exponentially larger sample of rollouts than the one simulated at a relatively little computational expense. The version of the Geiringer-like theorem presented in this paper greatly expands the variety of the similarity relations that can be used in practice, thereby significantly widening the range and the flexibility of applications.

Just as in \cite{MitavRowGeirNewMain}, a convenient way to represent a similarity relation $\mathcal{C}$ on a set of states $S$ is to assign a positive integer to each similarity set $O \in \mathcal{C}$ in a one-to-one fashion. Each element of a set $O$ labeled by an integer $l$ is then uniquely determined by an additional alphabet symbol. Unlike the case in \cite{MitavRowGeirNewMain}, it is possible for the same state to be labeled in a number of different ways as long as the corresponding integer labels differ. An example appears below.
\begin{ex}\label{setCoverEx}
A state space $S$ consisting of $13$ states and a set cover $\mathcal{C}$ consisting of $7$ sets are pictured in figure $1$. Notice that $\mathbf{1} = \{1a, \, 1b, \, 1c\}$, $\mathbf{3} = \{3a, \, 3b, \, 3c, \, 3d\}$, $\mathbf{2} = \mathbf{1} \cap \mathbf{3} = \{1a\} = \{2a\} = \{3a\}$, $\mathbf{5} = \{5a, \, 5b, \, 5c\}$, $\mathbf{6} = \{6a, \, 6b, \, 6c\}$, $\mathbf{4} = \mathbf{3} \cap \mathbf{6}$ consists of 2 elements that can be written as pictured on the Venn diagram on figure $1$ and, finally, $\mathbf{7} = \{7a, \, 7b, \, 7c\}$. For $\mathbf{i} \neq \mathbf{j}$ we have $\mathbf{i} \cap \mathbf{j} = \emptyset$ unless $\{i, \, j\} = \{1, \, 3\}, \, \{1, \, 2\}, \, \{3, \, 2\}, \, \{3, \, 6\}, \, \{3, \, 4\}$ or $\{4, \, 6\}$.
\end{ex}
\begin{figure}[ht]\label{setCoverFigure}
\centering
\includegraphics[height=2.5in, width=3in]{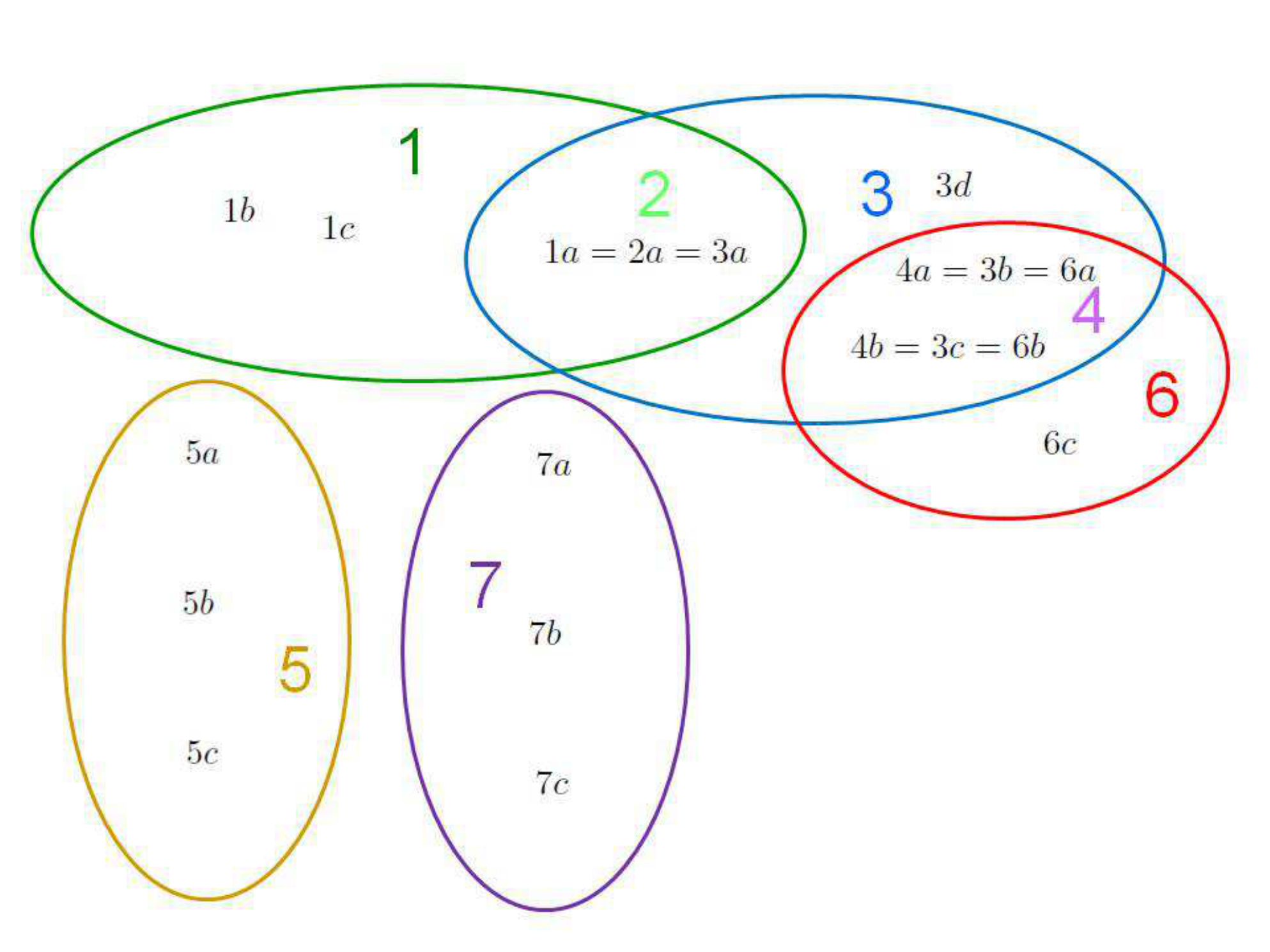}
\caption{A set cover of a 13-element set of states consisting of 7 sets.}
\end{figure}
We end this section with a couple of simple definitions, the significance of which will become very clear in section~\ref{GeirThmStSubsect}, when we state the main theorem of the current article.
\begin{defn}\label{transClosureDefn}
Given a set of state-action pairs $S$ and a set cover $\mathcal{C}$ of $S$, consider the relation $\backsim \, \subseteq S^2$ defined as $x \backsim y \, \Longleftrightarrow \, \exists \, O \in \mathcal{C}$ such that $x$ and $y \in O$ and notice that $\backsim$ is a symmetric and reflexive relation. We denote by $\overline{\mathcal{C}}$ the partition induced by the transitive closure of the relation $\backsim$, we will denote this equivalence relation (the transitive closure) by $\simeq$.
\end{defn}
\begin{rem}\label{equivRelRem}
Apparently, $$\mathcal{C} = \overline{\mathcal{C}} \Longleftrightarrow \mathcal{C} \text{ is a partition of } S \Longleftrightarrow \, \backsim = \simeq.$$
\end{rem}
An example below illustrates definition~\ref{transClosureDefn}
\begin{ex}\label{transClosureEx}
Continuing with example~\ref{setCoverEx}, consider the $13$-element set of state-action pairs and the set cover pictured on figure $1$. Then the symmetric relation
$$\backsim \, = \{(x, \, y) \, | \, i \in \mathbb{N} \text{ with } 1 \leq i \leq 7, \; x \in i \, \wedge \,  y \in i\},$$ the partition
$\overline{\mathcal{C}} = \{1 \cup 3 \cup 6, \, 5, \, 7\}$ and the equivalence relation $\simeq$ is the one corresponding to the partition $\overline{\mathcal{C}}$.
\end{ex}
\begin{defn}\label{EquivClassCompDef}
Given any set $O \in \mathcal{C}$, we will say that the unique equivalence class in $\overline{\mathcal{C}}$, call it $\overline{O}$, that contains $O$ is the expansion of the similarity set $O$.
\end{defn}
\begin{ex}\label{expEx}
Continuing with examples~\ref{setCoverEx} and \ref{transClosureEx}, we have $\overline{1} = \overline{2} = \overline{3} = \overline{4} = \overline{6} = 1 \cup 3 \cup 6$ while $\overline{5} = 5$ and $\overline{7} = 7$.
\end{ex}

\subsection{Rollouts and Recombination Operators}
\begin{defn}\label{treeRootedByChanceNode}
Suppose we are given a chance node $\vec{s} = (s, \vec{\alpha})$ and a sequence $\{\alpha_i\}_{i=1}^b$ of actions in $\vec{\alpha}$ (it is possible that $\alpha_i = \alpha_j$ for $i \neq j$). We may then call $\vec{s}$ a \emph{root state}, or a \emph{state in question}, the sequence $\{\alpha_i\}_{i=1}^b$, the \emph{sequence of moves (actions) under evaluation} and the set of moves $\mathcal{A} = \{\alpha \, | \, \alpha = \alpha_i$ for some $i$ with $1 \leq i \leq b\}$, the set of actions (or moves) under evaluation.
\end{defn}
\begin{defn}\label{RolloutDefn}
A \emph{rollout} with respect to the state in question $\vec{s} = (s, \vec{\alpha})$ and an action $\alpha \in \vec{\alpha}$ is a sequence of states following the action $\alpha$ and ending with a terminal label $f \in \Sigma$ where $\Sigma$ is an arbitrary set of labels\footnote{Intuitively, each terminal label in the set $\Sigma$ represents a terminal state that we can assign a numerical value to via a function $\phi: \, \Sigma \rightarrow \mathbb{Q}$. The reason we introduce the set $\Sigma$ of formal labels as opposed to requiring that each terminal label is a rational number straight away, is to avoid confusion in the upcoming definitions}, which looks as $\{(\alpha, \, s_1, \, s_2, \ldots, s_{t-1}, \, f)\}$. For technical reasons which will become obvious later we will also require that $s_i \neq s_j$ for $i \neq j$ (it is possible and common to have $s_i \overset{O}{\backsim} s_j$ for various $O \in \mathcal{C}$ though). We will say that the total number of states in a rollout (which is $k-1$ in the notation of this definition) is the \emph{height} of the rollout.
\end{defn}
\begin{rem}\label{rolloutDefRem}
Notice that in definition~\ref{RolloutDefn} we included only the initial move $\alpha$ made at the state in question (see definition~\ref{treeRootedByChanceNode}) which is the move under evaluation (see definition~\ref{treeRootedByChanceNode}). The moves between the intermediate states are chosen randomly or with respect to some dynamically updated distributions and are not evaluated so that there is no reason to consider them.
\end{rem}
%\begin{rem}\label{crossoverConvRepresRem}
%In subsection~\ref{equivSimSect} we have introduced a convenient notation for states to emphasize their respective equivalence classes. With such notation a typical rollout would appear as a sequence $\{(\alpha, \, (i_1, \, a_1), \, (i_2, \, a_2), \ldots, (i_{t-1}, a_{t-1}), \, f)\}$ with $i_j \in \mathbb{N}$ while $a_i \in A$. According to the requirement in definition~\ref{RolloutDefn}, $i_j = i_k$ for $j \neq k \, \Longrightarrow a_k \neq a_j$.
%\end{rem}
A single rollout provides rather little information about an action particularly due to the combinatorial explosion in the branching factor of possible moves of the player and the opponents. Normally a large, yet comparable with total resource limitations, number of rollouts is thrown to evaluate the actions at various positions. The challenging question which the current work addresses is how one can take full advantage of the parallel sequence of rollouts. Since the main idea is motivated by Geiringer theorem which is originated from population genetics (\cite{GeirOrigion}) and later has also been involved in evolutionary computation theory (\cite{PoliGeir}, \cite{MitavRowGeirMain} and \cite{MitavRowGeirGenProgr}) we shall exploit the terminology of the evolutionary computation community here.
\begin{defn}\label{popOfRolloutsDefn}
Given a state in question $\vec{s} = (s, \vec{\alpha})$ and a sequence $\{\alpha_i\}_{i=1}^b$ of moves under evaluation (in the sense of definition~\ref{treeRootedByChanceNode}) then a \emph{population} $P$ with respect to the state $\vec{s} = (s, \vec{\alpha})$ and the sequence $\{\alpha_i\}_{i=1}^b$ is a sequence of rollouts $P = \{r_i^{l(i)}\}_{i=1}^b$ where $r_i = \{(\alpha_i, \, s_1^i, \, s_2^i, \ldots, s^i_{l(i)-1}, \, f_i)\}$. Just as in definition~\ref{RolloutDefn} we will assume that $s_k^i \neq s_q^j$ whenever $i \neq j$ (which, in accordance with definition~\ref{RolloutDefn}, is as strong as requiring that $s_k^i \neq s_q^j$ whenever $i \neq j$ or $k \neq q$)\footnote{The last assumption that all the states in a population are formally distinct (although they may be similar with respect to various sets in $\mathcal{C}$) will be convenient later to extend the crossover operators from pairs to the entire populations. This assumption does make sense from the intuitive point of view as well since the exact state in most games involving randomness or incomplete information is simply unknown.} Moreover, we also assume that the terminal labels $f_i$ are also all distinct within the same population, i.e. for $i \neq j$ the terminal labels $f_i \neq f_j$\footnote{This assumption does not reduce any generality since one can choose an arbitrary (possibly a many to one) assignment function $\phi: \Sigma \rightarrow \mathbb{Q}$, yet the complexity of the statements of our main theorems will be mildly alleviated.} In a very special case when $(\exists \, O \in \mathcal{C}$ such that $s_j^i \overset{O}{\backsim} s_k^q) \Longrightarrow (j=k)$ we will say that the population $P$ is \emph{homologous}. Loosely speaking, a homologous population is one where similar states can not appear at different ``heights".
\end{defn}
\begin{rem}\label{popOfRolloutsRem}
Each rollout $r_i^{l(i)}$ in definition~\ref{popOfRolloutsDefn} is started with the corresponding move $\alpha_i$ of the sequence of moves under evaluation (see definition~\ref{treeRootedByChanceNode}). It is clear that if one were to permute the rollouts without changing the actual sequences of states the corresponding populations should provide identical values for the corresponding actions under evaluation. In fact, most authors in evolutionary computation theory (see \cite{VoseM}, for instance) do assume that such populations are equivalent and deal with the corresponding equivalence classes of multisets corresponding to the individuals (these are sequences of rollouts). Nonetheless, when dealing with finite-population Geiringer-like theorems it is convenient for technical reasons (see also \cite{MitavRowGeirMain} and \cite{MitavRowGeirGenProgr}) to assume the \emph{ordered multiset model} i.e. the populations are considered formally \emph{distinct} when the individuals are permuted. Incidentally, ordered multiset models are useful for other types of theoretical analysis in \cite{ShmittL1} and \cite{ShmittL2}.
\end{rem}
\begin{ex}\label{popEx}
An example of a population of rollouts from the state space $S$ described in example~\ref{setCoverEx} with the set cover pictured in figure $1$ appears in figure $2$.
\end{ex}
\begin{figure}[ht]\label{populationFigure}
\centering
\includegraphics[height=2.8in, width=3.8in]{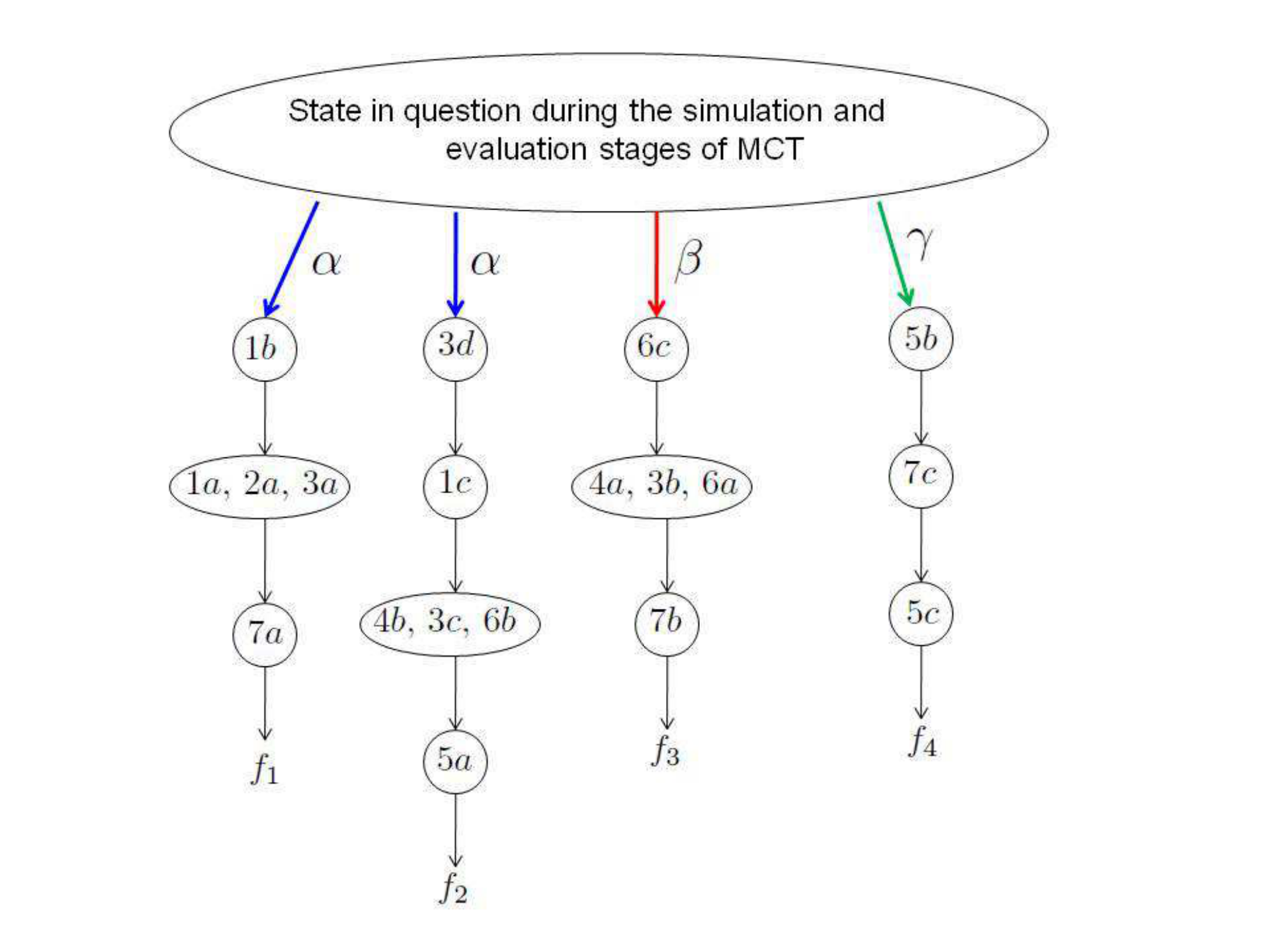} 
\caption{A population of 4 rollouts.}
\end{figure}

The main idea is that the actions taken at similar states should be interchangeable with some probability depending on the similarity level for a number of reasons such as incomplete information or simply because they are randomly explored during the simulation stage of the MCT algorithm. In the language of evolutionary computing, such a swap of moves is called a crossover. Due to randomness or incomplete information (together with the various similarity relations which can be defined using the expert knowledge of a specific game being analyzed) in order to obtain the most out of a sample (population in our language) of the parallel rollouts it is desirable to explore all possible populations obtained by making various swaps of the corresponding rollouts in similar positions. Computationally this task seems expensive if one were to run the type of genetic programming described precisely below, yet, it turns out that we can predict exactly what the limiting outcome of this ``mixing procedure" would be.\footnote{In this paper we will need to ``inflate" the population first and then take the limit of a sequence of these limiting procedures as the inflation factor increases. All of this will be rigorously presented and discussed in subsection~\ref{GeirThmStSubsect}; see also \cite{MitavRowGeirNewMain}.} We now continue with the rigorous definitions of crossover.
Crossover or recombination operators will be defined in terms of the similarities induced by the set cover much in the same way as it has been done in \cite{MitavRowGeirNewMain}. Nonetheless, here we are not limited to a single equivalence relation on the set of states, and many more recombination operators will be introduced depending on various similarity sets in $\mathcal{C}$. For this reason it is convenient to introduce the following notion:
\begin{defn}\label{totalSetContainmentDef}
Given a state $\vec{s} \in S$, we say that the collection of sets $\vec{s}(\mathcal{C}) = \{O \, | \, O \in \mathcal{C} \text{ and } \vec{s} \in O\}$ is the \emph{collection of similarity sets of the state} $\vec{s}$. Given states $\vec{s}_1$ and $\vec{s}_2 \in S$ and a set $O \in \vec{s}_1(\mathcal{C}) \cap \vec{s}_2(\mathcal{C})$ we say that $(O, \, \vec{s}_1, \, \vec{s}_2)$ is a \emph{recombination-compatible triple}.
\end{defn}
\begin{ex}\label{recombCompTupleEx}
Continuing with example~\ref{setCoverEx}, for the states $1b$ and $3a$ we have $1b(\mathcal{C}) = \{\mathbf{1}\}$ and $3a(\mathcal{C}) = \{\mathbf{1}, \, \mathbf{2}, \, \mathbf{3}\}$ so that $(\mathbf{1}, \, 1b, \, 3a)$ is a recombination-compatible triple while $(\mathbf{2}, \, 1b, \, 3a)$ is not.
\end{ex}
For every recombination compatible triple we can introduce the following two recombination operators:
\begin{defn}\label{rolloutPartCrossDefn}
Given two rollouts $$r_1 = (\alpha_1, \, \vec{s}_1, \, \vec{s}_2, \ldots, \vec{s}_{l(1)-1}, \, f)$$ and $$r_2 = (\alpha_2, \, \vec{t}_1, \, \vec{t}_2, \ldots, \vec{t}_{l(2)-1}, \, g)$$ of lengths $k(1)$ and $k(2)$ respectively that share no state in common (i.e., as in definition~\ref{RolloutDefn}, ) there are two (non-homologous) crossover (or recombination) operators we introduce here. For a recombination-compatible triple $(O, \, \vec{u}, \, \vec{v})$ define the \emph{one-point non-homologous crossover} transformation as follows: $\chi_{O, \, \vec{u}, \, \vec{v}}(r_1, \, r_2) = (\widehat{r}_1, \, \widehat{r}_2)$ where $$\widehat{r}_1 = (\alpha_1, \, \vec{s}_1, \, \vec{s}_2, \ldots, \vec{s}_{k-1}, \, \vec{t}_q, \, \vec{t}_{q+1}, \ldots, \vec{t}_{l(2)-1}, \, g)$$ and $$\widehat{r}_2 = (\alpha_2, \, \vec{t}_1, \, \vec{t}_2, \ldots, \vec{t}_{q-1}, \, \vec{s}_k, \, \vec{s}_{k+1}, \ldots, \vec{s}_{l(1)-1}, \, f)$$ if [either ($\vec{s}_k = \vec{u}$ and $\vec{t}_q = \vec{v}$) or vise versa: ($\vec{s}_k = \vec{v}$ and $\vec{t}_q = \vec{u}$)] and $(\widehat{r}_1, \, \widehat{r}_2) = (r_1, \, r_2)$ otherwise.

Likewise, we introduce a \emph{single position swap crossover}: \\$\nu_{O, \, \vec{u}, \, \vec{v}}(r_1, \, r_2) = (\widetilde{r}_1, \, \widetilde{r}_2)$ where
$$\widetilde{r}_1 = (\alpha_1, \, \vec{s}_1, \, \vec{s}_2, \ldots, \vec{s}_{k-1}, \, \vec{t}_q, \, \vec{s}_{k+1}, \ldots, \vec{s}_{l(1)-1}, \, f)$$ while
$$\widetilde{r}_2 = (\alpha_2, \, \vec{t}_1, \, \vec{t}_2, \ldots, \vec{t}_{q-1}, \, \vec{s}_k, \, \vec{t}_{q+1}, \ldots, \vec{t}_{l(2)-1}, \, g)$$ if [either ($\vec{s}_k = \vec{u}$ and $\vec{t}_q = \vec{v}$) or vise versa: ($\vec{s}_k = \vec{v}$ and $\vec{t}_q = \vec{u}$)] and $(\widehat{r}_1, \, \widehat{r}_2) = (r_1, \, r_2)$ otherwise.

In addition, a singe swap crossover is defined not only on the pairs of rollouts but also on a single rollout swapping the positions of the $O$-similar states $\vec{u}$ and $\vec{v}$ in the analogous manner:
If $$r = (\alpha, \, \vec{s}_1, \ldots, \vec{s}_{j-1}, \, \vec{s}_j, \, \vec{s}_{j+1}, \ldots, \vec{s}_{k-1}, \, \vec{s}_k, \, \vec{s}_{k+1}, \ldots, \vec{s}_{l-1}, \, f)$$ and [either ($\vec{s}_k = \vec{u}$ and $\vec{s}_q = \vec{v}$) or vise versa: ($\vec{s}_k = \vec{v}$ and $\vec{s}_q = \vec{u}$)] then $$\nu_{O, \, \vec{u}, \, \vec{v}}(r) = (\alpha, \, \vec{s}_1, \ldots, \vec{s}_{j-1}, \, \vec{s}_k, \, \vec{s}_{j+1}, \ldots$$$$ \ldots, \vec{s}_{k-1}, \, \vec{s}_j, \, \vec{s}_{k+1}, \ldots, \vec{s}_{l-1}, \, f$$ if [either ($\vec{s}_j = \vec{u}$ and $\vec{s}_k = \vec{v}$) or vise versa: ($\vec{s}_j = \vec{v}$ and $\vec{t}_k = \vec{u}$)] and, of course, $\nu_{O, \, \vec{u}, \, \vec{v}}(r)$ fixes $r$ (i.e. $\nu_{O, \, \vec{u}, \, \vec{v}}(r) = r$) otherwise.

%Suppose $i_k = j_q$ for some $k$ and $q$ with $1 \leq k < t(1)$ and $1 \leq q < t(2)$ (recall from section~\ref{equivSimSect} that this means the corresponding states at positions $k$ and $q$ are equivalent). We then define a \emph{one-point non-homologous crossover} of $r_1$ and $r_2$ to be the pair $\chi(r_1, \, r_2) = (t_1, \, t_2)$ where $t_1 = (\alpha_1, \, (i_1, \, a_1), \, (i_2, \, a_2), \ldots, (i_{k-1}, \, a_{k-1}), \, (j_q, \, b_q), \, (j_{q+1}, \, b_{q+1}), \ldots, (j_{t(2)-1}, b_{t(2)-1}), \, g)$ and $t_2 = (\alpha_2, \, (j_1, \, b_1), \, (j_2, \, b_2), \ldots, (j_{q-1}, \, b_{q-1}), \, (i_k, \, a_k), \, (i_{k+1}, \, a_{k+1}), \ldots, (i_{t(1)-1}, a_{t(1)-1}), \, f)$. Likewise, we introduce a \emph{single position swap crossover} $\nu(r_1, \, r_2) = (v_1, \, v_2)$ where $v_1 = (\alpha_1, \, (i_1, \, a_1), \, (i_2, \, a_2), \ldots, (i_{k-1}, \, a_{k-1}), \, (j_q, \, b_q), \, (i_{k+1}, \, a_{k+1}), \ldots, (i_{t(1)-1}, a_{t(1)-1}), \, f)$ while $v_2 = (\alpha_2, \, (j_1, \, b_1), \, (j_2, \, b_2), \ldots, (j_{q-1}, \, b_{q-1}), \, (i_k, \, a_k), \, (j_{q+1}, \, b_{q+1}), \ldots, (j_{t(2)-1}, b_{t(2)-1}), \, g)$. Moreover, The single position swap crossover also applies to a single rollout $r = (\alpha, \, (i_1, \, a_1), \, (i_2, \, a_2), \ldots, (i_{t(1)-1}, a_{t(1)-1}), \, f)$ defined as $\nu_{\{(i_k, \, a_k), (i_j, a_j)\}}(r) = r$ unless
\end{defn}
\begin{rem}\label{distConvRem}
Notice that definition~\ref{rolloutPartCrossDefn} makes sense thanks to the assumption that no rollout contains an identical pair of states in definition~\ref{RolloutDefn}.
\end{rem}
\begin{rem}\label{motivationRem}
Intuitively, performing one point crossover means that the corresponding player might have changed their strategy in a similar situation due to randomness and a single swap crossover corresponds to the player not knowing the exact state they are in due to incomplete information, for instance.
\end{rem}
Just as in case of defining crossover operators for pairs of rollouts, thanks to the assumption that all the states in a population of rollouts are formally distinct (see definition~\ref{popOfRolloutsDefn}), it is easy to extend definition~\ref{rolloutPartCrossDefn} to the entire populations of rollouts. In view of remark~\ref{motivationRem}, to get the most informative picture out of the sequence of parallel rollouts one would want to run the genetic programming routine without selection and mutation and using only the crossover operators specified above for as long as possible and then, in order to evaluate a certain move $\alpha$, collect the weighted average of the terminal values (i. e. the values assigned to the terminal labels via some rational-valued assignment function) of all the rollouts starting with the move $\alpha$ which ever occurred in the process. We now describe precisely what the process is.
% and give an example.

\begin{defn}\label{recombActOnPopsDef}
Given a population $P$ and a transformation of the form $\chi_{O, \, \vec{u}, \, \vec{v}}$, there exists at most one pair of distinct rollouts in the population $P$, namely the pair of rollouts $r_1$ and $r_2$ such that the state $\vec{u}$ appears in $r_1$ and the state $\vec{v}$ appears in $r_2$. If such a pair exists, then we define the recombination transformation $\chi_{O, \, \vec{u}, \, \vec{v}}(P) = P'$ where $P'$ is the population obtained from $P$ by replacing the pair of rollouts $(r_1, \, r_2)$ with the pair $\chi_{O, \, \vec{u}, \, \vec{v}}(r_1, \, r_2)$ as in definition~\ref{rolloutPartCrossDefn}. In any other case we do not make any change, i.e. $\chi_{O, \, \vec{u}, \, \vec{v}}(P) = P$. The transformation $\nu_{O, \, \vec{u}, \, \vec{v}}(P)$ is defined in an entirely analogous manner with one more amendment: if the states $\vec{u}$ and $\vec{v}$ appear within the same individual (rollout), call it $$r = (\alpha, \, \vec{s}_1, \ldots, \vec{s}_2, \ldots, \vec{u}, \ldots, \vec{v}, \ldots, \vec{s}_{l-1}, \, f),$$ and the state $\vec{u}$ precedes the state $\vec{v}$, then these states are interchanged obtaining the new rollout $$r' = (\alpha, \, \vec{s}_1, \ldots, \vec{s}_2, \ldots, \vec{v}, \ldots, \vec{u}, \ldots, \vec{s}_{l-1}, \, f).$$ Of course, it could be that the state $v$ precedes the state $u$ instead, in which case the definition would be analogous: if $$r = (\alpha, \, \vec{s}_1, \ldots, \vec{s}_2, \ldots, \vec{v}, \ldots, \vec{u}, \ldots, \vec{s}_{l-1}, \, f)$$ then replace the rollout $r$ with the rollout $$r'=(\alpha, \, \vec{s}_1, \ldots, \vec{s}_2, \ldots, \vec{u}, \ldots, \vec{v}, \ldots, \vec{s}_{l-1}, \, f).$$
\end{defn}
\begin{ex}\label{crossoverEx}
Continuing with example~\ref{popEx}, applying the crossover transformation $\chi_{\mathbf{3}, \, 3d, \, 3b}$ to the population we obtain the population on figure $3$.
\end{ex}

\begin{figure}\label{popCross3d3bFigure}
\centering
\includegraphics[height=2.8in, width=3.8in]{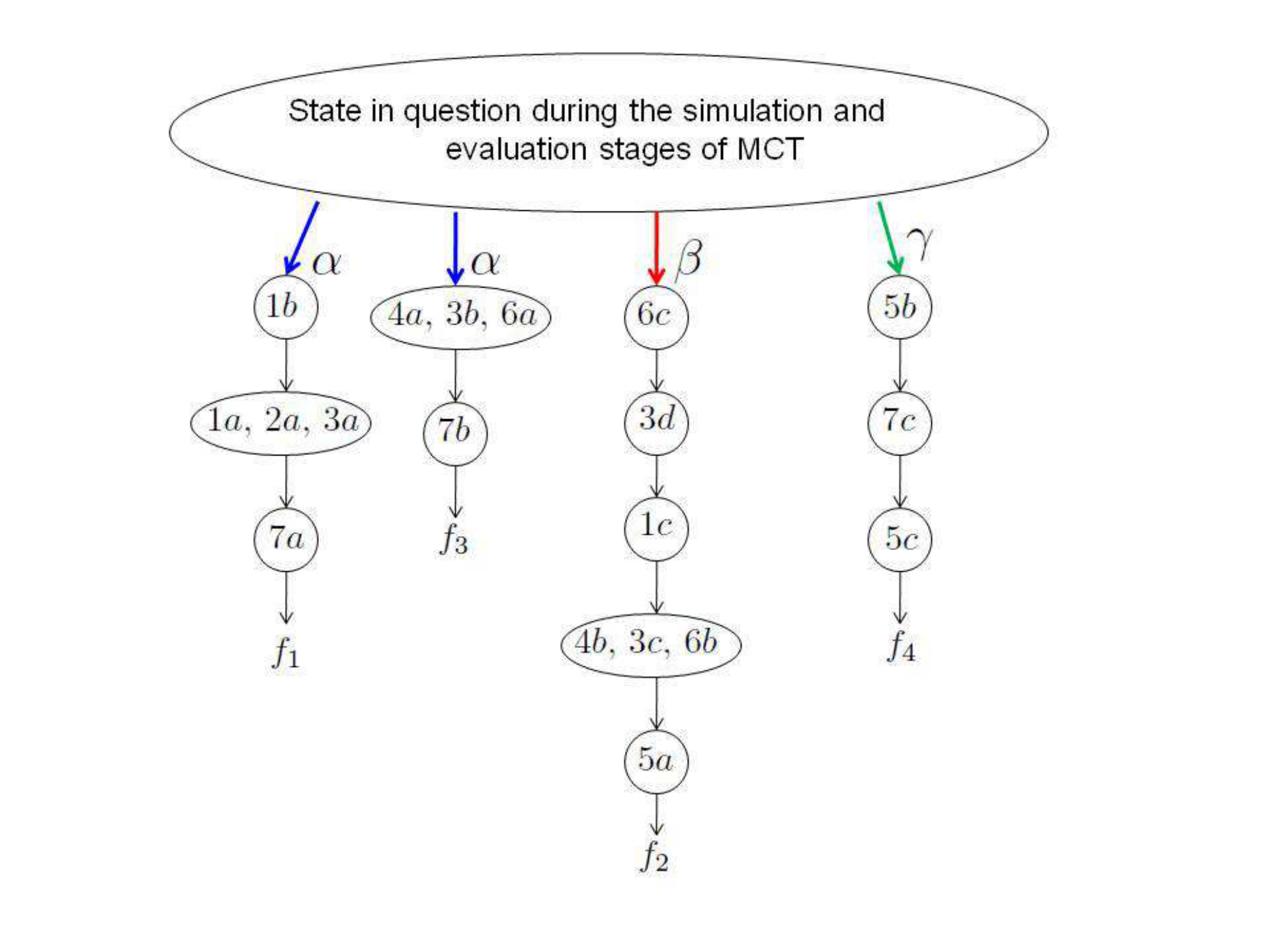} 
\caption{The population obtained after applying the crossover transformation in example~\ref{crossoverEx} to the population displayed in figure $2$.}
\end{figure}
\begin{rem}\label{BijectRem}
It is very important for the main theorem of our paper that each of the crossover transformations $\chi_{O, \, \vec{u}, \, \vec{v}}$ and $\nu_{O, \, \vec{u}, \, \vec{v}}$ is a bijection on their common domain, that is the set of all populations of rollouts at the specified chance node. As a matter of fact, the reader can easily verify by direct computation from definitions~\ref{recombActOnPopsDef} and \ref{rolloutPartCrossDefn} that each of the transformations $\chi_{O, \, \vec{u}, \, \vec{v}}$ and $\nu_{O, \, \vec{u}, \, \vec{v}}$ is an involution on its domain, i.e. $\forall \, O \in \mathcal{C}$ and $x, \, y \in O$ we have $\chi_{O, \, \vec{u}, \, \vec{v}}^2 = \nu_{O, \, \vec{u}, \, \vec{v}}^2 = \mathbf{1}$ where $\mathbf{1}$ is the identity transformation.
\end{rem}
\section{Geiringer and Geiringer-like  Theorems for POMDPs and MCT.}\label{mainIdeasSect}
\subsection{The Main Idea Of the Current Work}
Suppose a certain initial population of rollouts has been simulated during a simulation stage of the MCT. Although the rollouts have been simulated independently, similar states encountered during the simulations are likely to repeat in a number of settings.
%(such as, for instance, the French card game \'{E}cart\'{e} described in the previous subsection).
Assume now we were to run a GP (genetic programming) routine performing swaps \\(crossovers/recombinations) of rollouts in accordance with definitions~\ref{rolloutPartCrossDefn} and \ref{recombActOnPopsDef} without any selection or mutation. Clearly the swaps correspond to potential populations of rollouts that could have been simulated just as likely, provided that something different in the environment (and/or opponent's hand) has taken place. Intuitively speaking, if we were to run the GP longer and longer time, we would be getting significantly more enriched information about potential outcomes and hence improve the quality of the payoff estimates. The central idea behind results such as the main theorem of this article (a special case where the similarity relation is an equivalence relation (i.e. $\mathcal{C}$ is a partition of $S$) has been established in \cite{MitavRowGeirNewMain}) is that one can actually anticipate the long term (or limiting) frequency of occurrence of various rollouts provided such a genetic programming routine has ran. This type of predictions is what the Geiringer-like theorems are about. In \cite{MitavRowGeirMain} a rather general simple and powerful theorem (named ``finite population Geiringer theorem") has been established in the setting of Markov chains (populations being the states of the Markov chain: more on this in the next section) which tells us that under certain conditions that are satisfied by most recombination operators, the stationary distribution of this Markov chain is uniform. Furthermore, a methodology has been developed to derive what we call ``Geiringer-like" theorems that address the limiting frequency of occurrence of various schemata (in our case subsets of rollouts: more on this in the upcoming subsection~\ref{GeirThmStSubsect}). Based on such a theorem it is not hard to invent efficient parallel dynamic algorithms that estimate the expected action payoff values based on the sample obtained after the entire ``infinite time" run of the GP routine described above. It may be worthwhile to mention that ``homologous recombination" (translating into the setting of MCT this would mean that the similarity classes may occur only at the same heights of the corresponding rollouts) versions of a Geiringer-like theorem have been obtained previously in the setting of genetic programming using the methodology appearing in \cite{MitavRowGeirMain} (see \cite{MitavRowGeirGenProgr}). A version of Geiringer-like theorem with non-homologous recombination  remained an open question and it has been established recently in \cite{MitavRowGeirNewMain} in a very similar setting as in the current article. While the Geiringer-like theorem in \cite{MitavRowGeirNewMain} is already rather interesting and powerful, it is limited to the case when the notion of similarity is measured via an equivalence relation, which does not allow any degree of similarity: any two states are either similar or not, but there is no way to judge how similar they are and every state (with the corresponding actions) of the same similarity class is evaluated indistinguishably. In the current article, we point out that this limitation can be easily alleviated to allow practically any notion of similarity among the states (i.e. an arbitrary set cover of the set of states) and, at the same time, the statement and the proof of the corresponding Geiringer like theorem are somewhat simplified. In the next section we will establish a generalization of the finite population Geiringer theorem for POMDPs in \cite{MitavRowGeirNewMain} that will allow us to derive the corresponding generalization of the Geiringer-like theorem of \cite{MitavRowGeirNewMain}.
\subsection{Specializing the Finite Population Geiringer Theorem to the setting of Monte Carlo Sampling for POMDPs}\label{specializationAndMainGeiringerThmSubSect}
\begin{defn}\label{RecombStagePopTransDefn}
Let $\mathbf{n} = \{1, \, 2, \ldots, n\}$ denote the set of first $n$ natural numbers. Consider any probability distribution $\mu$ on the set of all finite sequences of crossover transformations $\mathcal{F} = \bigcup_{n=1}^{\infty} \mathcal{F}_n \cup \{\mathbf{1}\}$ where $$\mathcal{F}_n = (\{\chi_{O, \, \vec{u}, \, \vec{v}} \, | \, O \in \mathcal{C} \text{ and } \vec{u}, \, \vec{v} \in O\} \cup $$$$ \cup\{\nu_{O, \, \vec{u}, \, \vec{v}} \, | \, O \in \mathcal{C} \text{ and } \vec{u}, \, \vec{v} \in O\})^{\mathbf{n}}$$ which assigns a positive probability to the singleton sequences\footnote{This technical assumption may be altered in various manner as long as the induced Markov chain remains irreducible.} and to the \emph{identity element} $\mathbf{1}$. (i.e. to every element of the subset $\mathcal{F}_1 \cup \{\mathbf{1}\}$. Given a sequence of transformations $\vec{\Theta} = \{\Theta_{O_j, \, \vec{u}_j, \, \vec{v}_j}\}_{j=1}^n$ where each $\Theta$ is either $\chi$ or $\nu$ (i.e. $\forall \, j$ either $\Theta_{O_j, \, \vec{u}_j, \, \vec{v}_j} = \chi_{O_j, \, \vec{u}_j, \, \vec{v}_j}$ or $\Theta_{O_j, \, \vec{u}_j, \, \vec{v}_j} = \nu_{O_j, \, \vec{u}_j, \, \vec{v}_j}$), consider the transformation $$\widetilde{\Theta} = \Theta_{O_n, \, \vec{u}_n, \, \vec{v}_n} \circ \Theta_{O_{n-1}, \, \vec{u}_{n-1}, \, \vec{v}_{n-1}} \circ \ldots \circ \Theta_{O_2, \, \vec{u}_2, \, \vec{v}_2} \circ \Theta_{O_1, \, \vec{u}_1, \, \vec{v}_1}$$ on the set of all populations starting at the specified chance node obtained by composing all the transformations in the sequence $\vec{\Theta}$. The identity element $\mathbf{1}$ stands for the identity map on the set of all possible populations of rollouts. Now define the Markov transition Matrix $M_{\mu}$ on the set of all populations of rollouts (see definition~\ref{popOfRolloutsDefn}) as follows: given populations $X$ and $Y$ of the same size $k$, the probability  of obtaining the population $Y$ from the population $X$ after performing a single crossover stage, $p_{X \rightarrow Y} = \mu(\mathcal{S}_{X \rightarrow Y})$ where $$\mathcal{S}_{X \rightarrow Y}=\{\Gamma \, | \, \Gamma \in \mathcal{F} \text{ and } T(\Gamma)(X) = Y\}$$ where $$T(\Gamma) = \begin{cases}
\widetilde{\Theta} & \text{ if } \Gamma = \vec{\Theta} \\
\text{The identity map} \text{ if } \Gamma = \mathbf{1}.
\end{cases}$$
\end{defn}
\begin{rem}\label{seqVsTransfRem}
Evidently the map $T: \mathcal{F} \rightarrow P^P$ introduced at the end of definition~\ref{RecombStagePopTransDefn} can be regarded as a random variable on the set $\mathcal{F}$ described at the beginning of definition~\ref{RecombStagePopTransDefn} where $P$ denotes the set of all populations of rollouts containing $k$ individuals so that $P^P$ is the set of all endomorphisms (functions with the same domain and codomain) on $P$ and the probability measure $\mu_T$ on $P^P$ is the ``pushforward" measure induced by $T$, i.e. $\mu_T(S) = \mu(T^{-1}(S))$.\footnote{The sigma algebra on $P^P$ is the one generated by $T$ with respect to the sigma-algebra that is originally chosen on $\mathcal{F}$, however in practical applications the sets involved are finite and so all the sigma-algebras can be safely assumed to be power sets.} To alleviate the complexity of verbal (or written) presentation we will usually abuse the language and use the set $\mathcal{F}$ in place of $P^P$ so that a transformation $F \in P^P$ is identified with the entire set $T^{-1}(F) \in \mathcal{F}$. For example, $$\text{if we write } \mu(\{F \, | \, F \in \mathcal{F} \text{ and } F(X) = Y\})$$$$ \text{we mean } \mu(\{\Gamma \, | \, \Gamma \in \mathcal{F} \text{ and } T(\Gamma)(X) = Y\}).$$ It may be worth pointing out that the set $T^{-1}$ is not necessarily a singleton, i.e. the map $T$ is usually not one-to-one (for instance, given any $O \in \mathcal{C}$ and any $\vec{u}, \, \vec{v} \in O$ the sequence of the form $\vec{\Theta} = (\chi_{O, \, \vec{u}, \, \vec{v}}, \, \chi_{O, \, \vec{u}, \, \vec{v}})$ or of the form $\vec{\Theta} = (\nu_{O, \, \vec{u}, \, \vec{v}}, \, \nu_{O, \, \vec{u}, \, \vec{v}})$ both induce identity transformation on the set of populations of rollouts. Indeed, $\chi_{O, \, \vec{u}, \, \vec{v}} \circ \chi_{O, \, \vec{u}, \, \vec{v}} = \nu_{O, \, \vec{u}, \, \vec{v}} \circ \nu_{O, \, \vec{u}, \, \vec{v}} = \mathbf{1}$ since performing a swap at identical positions twice brings back the original population of rollouts.)
\end{rem}
\begin{rem}\label{closureUnderConcatRem}
Notice that any concatenation of sequences in $\mathcal{F}$ (which is what corresponds to the composition of the corresponding functions) stays in $\mathcal{F}$. In other words, the family of maps induced by $\mathcal{F}$ is closed under composition.
\end{rem}
Of course, running the Markov process induced by the transition matrix in definition~\ref{RecombStagePopTransDefn} infinitely long is impossible, but fortunately one does not have to do it. The central idea of the current paper is that the limiting outcome as time goes to infinity can be predicted exactly using the Geiringer-like theory and the desired evaluations of moves can be well-estimated at rather little computational cost in most cases. As pointed out in remark~\ref{seqVsTransfRem} above, each of the transformations $\Theta_{O, \, \vec{u}, \, \vec{v}}$ is an involution and, in particular, is bijective. Therefore, every composition of these transformations is a bijection as well. We deduce, thereby, that the family $\mathcal{F}$ consists of bijections only (see remark~\ref{seqVsTransfRem}). The finite population Geiringer theorem (see \cite{MitavRowGeirMain}) now applies and tells us the following:
\begin{defn}\label{equivRelForMCTDefn}
Given populations $P$ and $Q$ of rollouts at a specified state in question as in definition~\ref{popOfRolloutsDefn}, we say that $P \sim Q$ if there is a transformation $F \in \mathcal{F}$ such that $Q = F(P)$.
\end{defn}
\begin{ex}\label{equivPopEx}
The populations $P$ and $Q$ and $R$ displayed in figures $2$, $3$ and $4$ respectively are equivalent in the sense of definition~\ref{equivRelForMCTDefn} since $Q = \chi_{\mathbf{3}, \, 3d, \, 3b}(P)$, $R = \nu_{\mathbf{5}, \, 5a, \, 5b} \circ \chi_{\mathbf{5}, \, 5a, \, 5b} \circ \chi_{\mathbf{4}, \, 4a, \, 4b}(Q)$
\end{ex}
\begin{figure}\label{equivPopFigure}
\centering
\includegraphics[height=2.8in, width=3.8in]{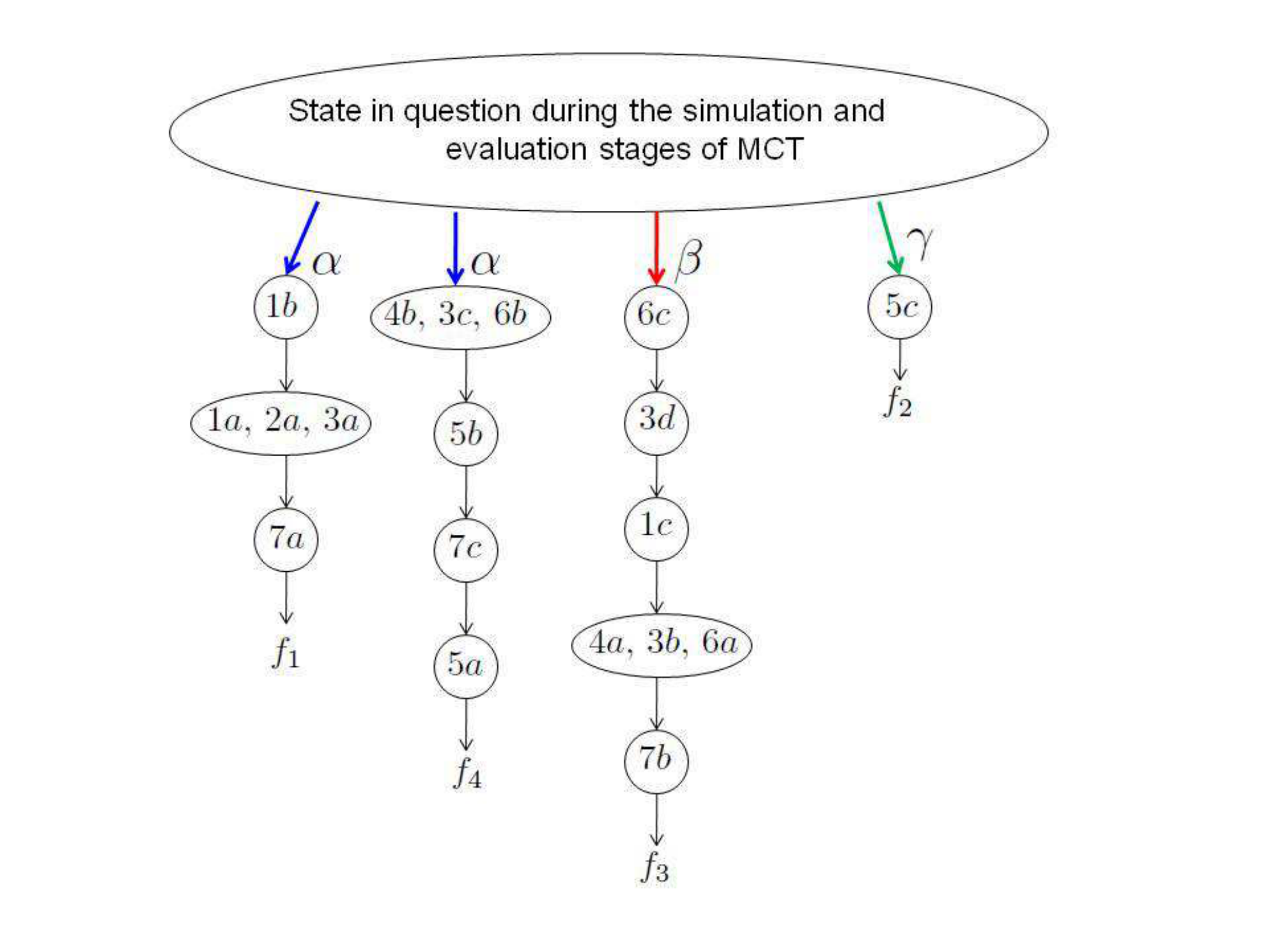} 
\caption{The population $R \sim Q \sim P$.}
\end{figure}
\begin{thm}[Geiringer Theorem for POMDPs]\label{GeirThmForMCTMain}
The relation $\sim$ introduced in definition~\ref{equivRelForMCTDefn} is an equivalence relation. Given a population $P$ of rollouts at a specified state in question, the restriction of the Markov transition matrix introduced in definition~\ref{RecombStagePopTransDefn} to the equivalence class $[P]$ of the population $P$ under $\sim$ is a well-defined Markov transition matrix which induces an irreducible and aperiodic Markov chain on $[P]$ and the unique stationary distribution of this Markov chain is the uniform distribution on $[P]$.\footnote{In fact, thanks to the application of the classical contraction mapping principle\footnote{This simple and elegant classical result about complete metric spaces lies in the heart of many important theorems such as the ``existence uniqueness" theorem in the theory of differential equations, for instance.} analyzed in section 6 of \cite{MitavRowGeirNewMain}) the stationary distribution is uniform in a rather strong sense as discussed in theorem 23 and example 24 of \cite{MitavRowGeirNewMain}.}
\end{thm}
Knowing that the limiting frequency of occurrence of a any two given populations $Q_1$ and $Q_2 \in [P]$ is the same, it is sometimes possible to compute the limiting frequency of occurrence of any specific rollout and even certain subsets of rollouts using the machinery developed in \cite{MitavRowGeirMain}, \cite{MitavRowGeirGenProgr} and enhanced further in \cite{MitavRowGeirNewMain}.

To state and derive these ``Geiringer-like" results we need to introduce the appropriate notions of schemata (see, for instance, \cite{Antonisse} and \cite{PoliSchema}) here.
\subsection{Schemata for MCT Algorithm}\label{schemataSubsectShallow}
\begin{defn}\label{schemaForMCTPopDefHolland}
Given a state $(s, \vec{\alpha})$ in question (see definition~\ref{treeRootedByChanceNode}), a rollout \emph{Holland-Poli schema} is a sequence consisting of entries from the set$\vec{\alpha} \cup \mathcal{C} \cup \{\#\} \cup \Sigma$ of the form $h = \{x_i\}_{i=1}^k$ for some $k \in \mathbb{N}$ such that for $k>1$ we have $x_1 \in \vec{\alpha}$, $x_i \in \mathcal{C}$ when $1 < i < k$ represents a similarity class of states, and $x_k \in \{\#\} \cup \Sigma$ could represent either a terminal label if it is a member of the set of terminal labels $\Sigma$, or any substring defining a valid rollout if it is a $\#$ sign.\footnote{This notion of a schema is somewhat of a mixture between Holland's and Poli's notions.} For $k=1$ there is a unique schema of the form $\#$. Every schema uniquely determines a set of rollouts
$R_h = \begin{cases}
\{(x_1, \, \vec{s}_2, \, \vec{s}_3, \ldots, \vec{s}_{k-1}, x_k) \, \\| \, \vec{s}_i \in x_i \text{ for } 2<i<k-1\} & \text{if } k>1 \text{ and } x_k \in \Sigma\\
\, & \,\\
\{(x_1, \, \vec{s}_2, \, \vec{s}_3, \ldots, \vec{s}_{k-1}, \, \\\vec{t}_k, \, \vec{t}_{k+1}, \ldots, f)\\
\, | \, \vec{s}_i \in x_i \text{ for } 2<i<k-1, \, \vec{t}_j \in S\} & \text{if } k > 1 \text{ and } x_k = \#\\
\, & \,\\
\text{the entire set of all possible rollouts} & \text{if } k = 1 \text{ or,} \\
\, & \text{equivalently, } h=\#.
\end{cases}$ which fit the schema in the sense mentioned above. We will often abuse the language and use the same word schema to mean either the schema $h$ as a formal sequence as above or schema as a set $R_h$ of rollouts which fit the schema. For example, if $h$ and $h^*$ is a schema, we will write $h \cap h^*$ as a shorthand notation for $R_h \cap R_{h^*}$ where $\cap$ denotes the usual intersection of sets. Just as in definition~\ref{RolloutDefn}, we will say that $k-1$, the number of states in the schema $h$, is the \emph{height} of the schema $h$.
\end{defn}
\begin{ex}\label{schemataEx}
Continuing with example~\ref{setCoverEx}, consider the rollout Holland-Poli schema $h = (\beta, \, 6, \, 3, \, 1, \, 4, \, \#)$. Then the $3^{\text{rd}}$ rollout in the population pictured in figure $4$  starting with the action $\beta$ fits the schema $h$ while the rollout $r = (\beta, \, 6c, \, 3a, \, 1c, \, 3a)$ does not since $3a = 1a = 2a \notin \mathbf{4}$ (see figure $1$).
\end{ex}
The notion of schema is useful for stating and proving Geiringer-like results largely thanks to the following notion of partial order.
\begin{defn}\label{schemaPosetDef}
Given schemata $h$ and $g$ we will write $h > g$ either if $h=\#$ and $g \neq \#$ or\\
$h=(x_1, \, x_2, \, x_3, \ldots, x_{k-1}, \, \#)$ while \\
$g = (x_1, \, x_2, \, x_3, \ldots, x_{k-1}, \, y_{k}, \, y_{k+1}, \ldots, y_{l-1}, \, y_l)$ where $y_l$ could be either of the allowable values: a $\#$ or a terminal label $f \in \Sigma$. However, if $y_l = \#$ then we require that $l > k$.
\end{defn}
An obvious fact following immediately from definitions~\ref{schemaForMCTPopDefHolland} and \ref{schemaPosetDef} is the following.
\begin{prop}\label{schemaPartOrderHollandShallow}
Suppose we are given schemata $h$ and $g$. Then $h \geq g \Longrightarrow S_h \supseteq S_g$.
\end{prop}
\begin{rem}\label{posetNonIsoRem}
It may be worth pointing out that the converse of proposition~\ref{schemaPartOrderHollandShallow} is false. Continuing with examples~\ref{setCoverEx} and \ref{schemataEx}, consider, for instance, schemata \\$h = (\beta, \, 6, \, 3, \, 1, \, 4, \, \#)$ and $g = (\beta, \, 6, \, 3, \, 1, \, 6, \, \#)$. Evidently $S_g \supset S_h$ yet the schemata $h$ and $g$ are incomparable in the sense of definition~\ref{schemaPosetDef}. Fortunately, it is only proposition~\ref{schemaPartOrderHollandShallow} (and not the converse) that's involved in deriving Geiringer-like theorems from theorem~\ref{GeirThmForMCTMain}.
\end{rem}
\subsection{The Statement of Geiringer-like Theorems for the POMDPs}\label{GeirThmStSubsect}
In evolutionary computation Geiringer-like results address the limiting frequency of occurrence of a set of individuals fitting a certain schema (see \cite{PoliGeir}, \cite{MitavRowGeirMain} and \cite{MitavRowGeirGenProgr}). In this work our theory rests on the finite population model based on stationary distribution of the Markov chain of all populations potentially encountered in the process (see theorem~\ref{GeirThmForMCTMain}). The ``limiting frequency of occurrence" (rigorous definition appear in \cite{MitavRowGeirNewMain}, \cite{MitavRowGeirMain} and \cite{MitavRowGeirGenProgr}; an informal description is provided prior to the statement of the theorem in the current article) of a certain subset of individuals determined by a Holland-Poli schema $h$ among all the populations in the equivalence class $[P]$ as time increases (i.e. as $t \rightarrow \infty$), where $P$ is the initial population of rollouts, will be expressed solely in terms of the initial population $P$ and schema $h$. These quantities are defined below.
\begin{defn}\label{popStateCountDownDef}
For any action under evaluation $\alpha$, let $\alpha \downarrow(P) = \{\overline{O} \, | \, \overline{O} \in \overline{\mathcal{C}}$ and at least one of the rollouts in the population $P$ fits the Holland-Poli schema $(\alpha, \, \overline{O}, \, \#)\}$. Likewise, for an equivalence class $\overline{O} \in \mathcal{C}$ define a set valued function on the populations of size $b$, as $\overline{O} \downarrow (P) = \{\overline{T} \, | \, \exists$ states $\vec{s} \in \overline{O}$, $\vec{t} \in \overline{T}$ and a rollout $r$ in the population $P$ such that $r = (\ldots, \vec{s}, \, \vec{t}, \ldots) \, \} \cup \{f \, | \, f \in \Sigma$ and $\exists$ a state $\vec{s} \in \overline{O}$ and a rollout $r$ in the population $P$ such that $r = (\ldots, \vec{s}, \, f) \, \}$. In words, $\overline{O} \downarrow (P)$ is the collection of all equivalence classes in $\overline{\mathcal{C}}$ together with the terminal labels which appear after the states from the equivalence class $\overline{O}$ in at least one of the rollouts from the population $P$.
%Finally, introduce one more function, namely $i \downarrow_\Sigma: \Omega^b \rightarrow \mathbb{N} \cup \{0\}$ by letting $i \downarrow_\Sigma(P) = |\{f \, | \, f \in \Sigma \cap i \downarrow (P)\}|$, that is, the total number of terminal labels (which are assumed to be all formally distinct for convenience) following the equivalence class $i$ in a rollout of the population $P$.
\end{defn}
As usual, we illustrate definition~\ref{popStateCountDownDef} with an example.
\begin{ex}\label{popDownRelExample}
Continuing with examples~\ref{setCoverEx} and \ref{popEx}, from figure $2$ we see that the states following the action $\alpha$ are of types $1$ and $3$ and both of these similarity classes belong to the equivalence class $1 \cup 3 \cup 6$ so that, according to definition~\ref{popStateCountDownDef}, $\alpha \downarrow \, (P) = \{1 \cup 3 \cup 6\}$. The only state following the action $\beta$ is of type $6$ and the expansion of $6$ is, again, $1 \cup 3 \cup 6$, so that $\beta \downarrow \, (P) = \{1 \cup 3 \cup 6\}$. The only state following the action $\gamma$ is of type $5$ and so $\gamma \downarrow \, (P) = \{\overline{5}\} = \{5\}$. Various states from every equivalence class in $\overline{\mathcal{C}}$ follow states of types $1$, $3$ and $6$: for instance, $1c$ follows $3d$ and $5a$ follows $3c$ in the second rollout, while $7a$ follows $1a$ in the first rollout. No terminal label follows the equivalence class $\{1 \cup 3 \cup 6\}$. It follows then that $1 \cup 3 \cup 6 \downarrow = \overline{\mathcal{C}}$. The only similarity class following states from the equivalence class $\overline{5} = 5$ is $7$. Terminal labels $f_2$ and $f_4$ follow states $5a$ in the second rollout and $5c$ in the $4^{\text{th}}$ rollout. It follows than that $5 \downarrow \, (P)= \{f_2, \, 7, \, f_4\}$. Analogously, $7 \downarrow \, (P) = \{f_1, \, f_3, \, 5\}$.
\end{ex}
%Prior to introducing the following couple of important definitions, to alleviate the notational complexity, we will assume throughout that the set of all states $S$ is precisely the set of states that appear within the initial population of the rollout samples.\footnote{This assumption does not reduce the generality since one can always reduce the state space to include only these states that appear within the simulated sample: indeed, these are the only states one has any information about until encountering new ones and dynamically updating the set of states.} Notice that the set of states in the population pictured on figure $2$ is precisely the set of states on the Venn diagram set cover in figure $1$.
\begin{defn}\label{PopStateOrderSeqDef}
Given a population $P$ and equivalence classes $\overline{O} \in \overline{\mathcal{C}}$ and  $\overline{T} \in \overline{O} \downarrow (P)$ let $\text{Order}(\overline{O} \downarrow \overline{T}, \, P) =$
$$= |\{((\overline{O}, a), \, (\overline{T}, \, b)) \, | \, \text{ the segment }((\overline{O}, a), \, (\overline{T}, \, b))$$$$\text{appears in one of the rollouts in the population }P\}|.$$

Loosely speaking, $\text{Order}(\overline{O} \downarrow \overline{T}, \, P)$ is the total number of times the equivalence class $\overline{T}$ follows the equivalence class $\overline{O}$ within the
population of rollouts $P$.

Let $$\text{Order}(\overline{O} \downarrow \, P) = \sum_{\overline{T} \in \overline{O} \downarrow \, P}\text{Order}(\overline{O} \downarrow \overline{T}, \, P) + |\Sigma \cap (\overline{O} \downarrow \overline{T}, \, (P))|$$ denote the total number of states and terminal labels that follow the states from the equivalence class $\overline{O}$ in the population $P$.

Likewise, given a population of rollouts $P$, an action $\alpha$ under evaluation and an equivalence class $\overline{T} \in \alpha \downarrow (P)$, let
$$\text{Order}(\alpha \downarrow \overline{T}, \, P) =|\{(\alpha, \, (\overline{T}, \, b)) \, | \, \text{ the segment }(\alpha, \, (\overline{T}, \, b))$$$$\text{ appears in one of the rollouts in the population }P\}|.$$
Alternatively, $\text{Order}(\alpha \downarrow \overline{T}, \, P)$ is the number of rollouts in the population $P$ fitting the rollout Holland schema $(\alpha, \, \overline{T}, \, \#)$. $$\text{Order}(\alpha \downarrow (P)) = \sum_{\overline{T} \in \overline{O} \downarrow \, P}\text{Order}(\overline{O} \downarrow \overline{T}, \, P)$$ denotes the total number of states that follow the action under evaluation $\alpha$ in the population $P$.
\end{defn}
%\begin{defn}\label{setPopIntersectDefn}
%For any subset $K \subseteq S$ of states (not necessarily a member of the set cover $\mathcal{C}$) and a population $P$ we'll write $K \cap_{\text{pop}} P$ to denote the set of all states in $K$ that appear within at least one rollout in the population $P$.
%\end{defn}
\begin{ex}\label{popStateOrderEx}
Continuing with example~\ref{popDownRelExample}, for the population $P$ pictured in figure $2$ and the set cover $\mathcal{C}$ displayed via a Venn diagram on figure $1$, recall that $\alpha \downarrow \, (P)= \beta \downarrow \, (P) = \{1 \cup 3 \cup 6\}$ and notice that the total number of states from the only equivalence class $1 \cup 3 \cup 6 \in \alpha \downarrow$ that follows the action $\alpha$ is $2$: the state $1a = 2a = 3a$ in the first rollout and the state $1c$ following the state $3d$ in the second rollout; while the total number of states in the only equivalence class $1 \cup 3 \cup 6 \in \alpha \downarrow$ that follows the action $\beta$ is $1$: the state $6c$ in the $3^{\text{rd}}$ rollout. Thus $\text{Order}(\alpha \downarrow 1 \cup 3 \cup 6, \, P)=\text{Order}(\alpha \downarrow (P)) =2$ and $\text{Order}(\beta \downarrow 1 \cup 3 \cup 6, \, P)=\text{Order}(\beta \downarrow (P)) =1$. Likewise, $\text{Order}(\gamma \downarrow 5, \, P)=\text{Order}(\gamma \downarrow \, (P))=1$. States from the equivalence class $1 \cup 3 \cup 6$ follow their own kind $4$ times in the population $P$: $1a = 2a = 3a$ follows $1b$ in the first rollout, $1c$ follows $3d$ and $4b = 3c = 6b$ follows $1c$ in the second rollout, while $4a=3b=6a$ follows $6c$ in the third rollout so that $\text{Order}(1 \cup 3 \cup 6 \downarrow 1 \cup 3 \cup 6, \, P)=4$. The only state in the equivalence class $5$ that follows a state from the equivalence class $1 \cup 3 \cup 6$ is $5a$ in the second rollout, following the state $4b = 3c = 6b \in 1 \cup 3 \cup 6$ so that $\text{Order}(1 \cup 3 \cup 6 \downarrow 5, \, P)=1$. There are $2$ states from the equivalence class $7$, namely $7a$ and $7b$, that follow the states from the equivalence class $1 \cup 3 \cup 6$, namely $1a = 2a =3a$ and $4a=3b=6a$ respectively, in the first and the third rollouts of the population $P$ so that $\text{Order}(1 \cup 3 \cup 6 \downarrow 7, \, P)=2$. No terminal label follows a state from the equivalence class $1 \cup 3 \cup 6$ and we deduce that $\text{Order}(1 \cup 3 \cup 6 \downarrow (P))=4+1+2 = 7$. Recalling that $5 \downarrow \, (P) = \{f_2, \, 7, \, f_4\}$ and observing that the only state in the equivalence class $7$ that follows a state from the equivalence class $5$, namely the state $5b$, is $7c$, while $|\Sigma \cap \{f_2, \, 7, \, f_4\}| = |\{f_2, \, f_4\}| = 2$ we deduce that $\text{Order}(5 \downarrow 7, \, P)=1$ and $\text{Order}(5 \downarrow \, (P))=1 + 2 = 3$. Likewise, the reader may compute the remaining numbers $\text{Order}(7 \downarrow 5, \, P)=1$ and $\text{Order}(7 \downarrow \, (P))=1 + 2 = 3$.
\end{ex}
Observe that applying any recombination transformation of the form $\chi_{O, \, \vec{s}, \, \vec{t}}$ or $\nu_{O, \, \vec{s}, \, \vec{t}}$ to a population $P$ of rollouts neither removes any states from the population nor adds any new ones, and hence another important invariance property of the equivalent populations that opens the door for a lovely application of Markov inequality in the proof of the main theorem of the current article (see \cite{MitavRowGeirNewMain}) that follows from the same considerations is stated below.
\begin{rem}\label{averageHightRem}
Given any population $Q \in [P]$, the total number of states in the population $Q$ is the same as that in the population $P$. Apparently, as we already mentioned, the the total number of states in a population is the sum of the heights of all rollouts in that population (see definition~\ref{RolloutDefn} and \ref{popOfRolloutsDefn}). It follows then, that the sum of the heights of all rollouts within a population is an invariant quantity under the equivalence relation in definition~\ref{equivRelForMCTDefn}. In other words, if $Q \sim P$ then the sum of the heights of the rollouts in the population $Q$ is the same as the sum of the heights of the rollouts in the population $P$.
\end{rem}
There is yet one more important notion, namely that of the ``limiting frequency of occurrence" of a schema as one runs the genetic programming routine with recombination only we need to introduce to state the Geiringer-like results of the current paper. A rigorous definition in the most general framework appears in \cite{MitavRowGeirNewMain}, \cite{MitavRowGeirMain} and \cite{MitavRowGeirGenProgr}. The description below is sufficient to understand the statement of finite population Geiringer-like theorems.

\emph{Informal Description of the Limiting Frequency of Occurrence:} Given a schema $h$ and a population $P$ of size $m$, suppose we run the Markov process $\{X_n\}_{n=0}^{\infty}$ on the populations in the equivalence class $[P]$ of the initial population of rollouts $P$ as in definition~\ref{RecombStagePopTransDefn}.\footnote{In fact, such Markov chains don't even have to be time homogenous: (see theorem 23 and example 24 of \cite{MitavRowGeirNewMain} for a detailed exposition.} As discussed previously, this corresponds to ``running the genetic programming routine forever" and each recombination models the changes in player's strategies due to incomplete information, randomness personality etc. Up to time $t$ a total of $m \cdot t$ individuals (counting repetitions) have been encountered. Among these a certain number, say $h(t)$, fit the schema $h$ in the sense of definition~\ref{schemaForMCTPopDefHolland}. We now let $\Phi(P, \, h, \, t) = \frac{h(t)}{m \cdot t}$ to be the proportion of these individuals fitting the schema $h$ out of the total number of individuals encountered up to time $t$. Although it may be possible to derive the formulas for $\lim_{t \rightarrow \infty} \Phi(P, \, h, \, t)$ in the most general case when the initial population of rollouts $P$ is non-homologous\footnote{such exact formulas have been derived for the case of homologous recombination in the setting of GP: see \cite{MitavRowGeirGenProgr}} (in other words when the states representing the same equivalence class may appear at various ``heights" in the same population of rollouts: see definition~\ref{popOfRolloutsDefn}), the formulas obtained in this manner would definitely be significantly more cumbersome and would not be as well suited for algorithm development\footnote{This is an open question, yet it's practical importance is highly unclear} as the limiting result with respect to ``inflating" the initial population $P$ in the sense described below. Remarkably, the formula for the limiting result in the general non-homologous initial population case coincides with the one for the homologous populations.
\begin{defn}\label{popInflationDef}
Given a population $P = \{r_i^{l(i)}\}_{i=1}^b$ of rollouts in the sense of definition~\ref{popOfRolloutsDefn}, where $$r_i = \{(\alpha_i, \, \vec{s}_1^i, \, \vec{s}_2^i, \ldots, \vec{s}_{l(i)-1}^i, \, f_i)\}$$ and a positive integer $m$, we first increase the size of the set of states $S$ by a factor of $m$: formally, we update the set of states as follows:
$$S := S \times m = \{(\vec{s}, \, i) \, | \, \vec{s} \in S, \, i \in \mathbb{N} \text{ and } 1 \leq i \leq m\}.$$ Likewise, we also increase the terminal set of labels $\Sigma$ by a factor of $m$ so that $$\Sigma := \Sigma \times m = \{(f, \, i) \, | \, f \in \Sigma, \, i \in \mathbb{N} \text{ and } 1 \leq i \leq m\}.$$ Certainly, the set cover $\mathcal{C}$ of the original set $S$ is naturally updated into the corresponding set cover of the updated set of states as $$\mathcal{C} := \{O \, | O := \{(\vec{s}, \, i) \, | \, \vec{s} \in O, \, i \in \mathbb{N} \text{ and } 1 \leq i \leq m\}\}$$ Now we let $$P_m = \{r_{i, \, k}^{l(i)}\}_{1 \leq i \leq b \text{ and }1 \leq k \leq m}$$ where $$r_{i, \, k}^{l(i)} = \{(\alpha_i, \, (\vec{s}_1^i, \, k), \, (\vec{s}_2^i, \, k), \ldots, (\vec{s}^i_{l(i)-1}, \, k), \, (f_i, k))\}.$$ We will say that the population $P_m$ is an \emph{inflation} of the population $P$ by a factor of $m$.
\end{defn}
Essentially, a population $P_m$ consists of $m$ \emph{formally distinct} copies of each rollout in the population $P$. Intuitively speaking, the stochastic information captured in the sample of rollouts comprising the population $P_m$ (such as the frequency of obtaining a state in a similarity set $O$ after a state in a similarity set $T$) is the same as the one contained within the population $P$ emphasized by the factor of $m$. In fact, the following rather important obvious facts make some of this intuition precise:
\begin{prop}\label{popRatioFacts}
Given a population $P$ of rollouts and a positive integer $m$ consider the inflation of the population $P$ by a factor of $m$, $P_m$ as in definition~\ref{popInflationDef}. Then the following are true:
$$\alpha \downarrow (P_m) = \alpha \downarrow (P), \; \overline{O} \downarrow (P_m) = \overline{O} \downarrow (P)$$  while
%$$\emph{Order}(\alpha \downarrow O, \, P_m) = m \cdot \emph{Order}(\alpha \downarrow O, \, P),$$
$$\emph{Order}(\overline{O} \downarrow \overline{T}, \, P_m) = m \cdot \emph{Order}(\overline{O} \downarrow \overline{T}, \, P),$$
$$\emph{Numb}(\alpha, \, P_m) = m \cdot \emph{Numb}(\alpha, \, P)$$
and
\begin{equation}\label{inflatedQuantitiesEq}
\emph{Order}(\overline{O} \downarrow, \, P_m) = m \cdot \emph{Order}(\overline{O} \downarrow, \, P)
\end{equation}
For any population of rollouts $Q$ let $\emph{Total}(Q)$ denote the total number of states in the population $Q$ which is, of course, the same thing as the sum of the heights
of all rollouts in the population $Q$. Then, clearly, $\emph{Total}(P_m) = m \cdot \emph{Total}(P)$. In the special case when $P$ is a homologous population, $\forall \, m \in \mathbb{N}$ so is the population $P_m$.
\end{prop}
\begin{rem}\label{terminalHollandShemataRem}
We will use the same Holland-Poli schema $h = (\alpha, O_1, \, O_2, \ldots, O_{k-1}, \, f)$ to model the corresponding sets of rollouts of the set of states $S$ as of the set $S$ inflated by a factor of $m$. Of course, $S_h = S^m_h$ where $S^m_h$ is the set of states inflated by a factor of $m$ just as in proposition~\ref{popRatioFacts} above. Furthermore, it is important to point out that in the statement of the theorem below we use the sets $O$ from the initial (non-inflated, or, alternatively, inflated by a factor of $1$) population to express the limiting frequency of occurrence.
\end{rem}
We are finally ready to state the main result of the current paper.
\begin{thm}[The Geiringer-Like Theorem for MCT]\label{GeiringerLikeThmForMCTMain}
Repeat verbatim the assumptions of theorem~\ref{GeirThmForMCTMain}\footnote{This theorem holds based on the same argument under the assumption of a more general theorem 23 of \cite{MitavRowGeirNewMain}}. Let $$h = (\alpha, \, O_1, \, O_2, \ldots, O_{k-1}, x_k)$$ where $x_k \in \{\#\} \cup \Sigma$ be a given Holland-Poli schema. For $m \in \mathbb{N}$ consider the random variable $\Phi(P_m, \, h, \, t)$ described in the paragraph just above (alternatively, a rigorous definition in the most general framework appears in \cite{MitavRowGeirNewMain}, \cite{MitavRowGeirMain} and \cite{MitavRowGeirGenProgr}) with respect to the Markov process $X_n^m$ where $m$ indicates that the initial population of rollouts is the inflated population $P_m$ as in definition~\ref{popInflationDef}. Then
$$\lim_{m \rightarrow \infty}\lim_{t \rightarrow \infty}\Phi(P_m, \, h, \, t) = \frac{\text{Numb}(\alpha, \, P)}{b} \cdot \prod_{q=1}^{k-1}\frac{|O_q|}{|\overline{O_q}|} \times $$
\begin{equation}\label{GeiringerThmMainEq}
\frac{\emph{Order}(\alpha \downarrow \overline{O_1}, \, P)}{\emph{Order}(\alpha \downarrow, \, P)} \cdot \prod_{q = 1}^{k-1}\frac{\emph{Order}(\overline{O_{q-1}} \downarrow \overline{O_q}, \, P)}{\emph{Order}(\overline{O_{q-1}} \downarrow, P)} \cdot \text{\emph{LF}(P, h)}
\end{equation}
where $$\text{\emph{LF}(P, h)} = \begin{cases}
1 & \text{if } x_k = \#\\
0 & \text{if } x_k = f \in \Sigma \text{ and } f \notin \overline{O_{k-1}} \downarrow (P)\\
%\; & \\ %\; & \text{rollout in the population }P\\
\emph{Fraction} & \text{if } x_k = f \in \Sigma \text{ and } f \in \overline{O_{k-1}} \downarrow_{\Sigma} (P)
\end{cases}$$ where $$\emph{Fraction} = \frac{1}{\emph{Order}(\overline{O_{k-1}} \downarrow, P)}$$ (we write ``LF" as short for ``Last Factor"). Furthermore, in the special case when the initial population $P$ is homologous (see definition~\ref{popOfRolloutsDefn}), one does not need to take the limit as $m \rightarrow \infty$ in the sense that $\lim_{t \rightarrow \infty}\Phi(P_m, \, h, \, t)$ is a constant independent of $m$ and its value is given by the right hand side of equation~\ref{GeiringerThmMainEq}.

%\footnote{The case of homologous recombination has been established in a different but mathematically equivalent framework in \cite{MitavRowGeirMain} and \cite{MitavRowGeirGenProgr} nonetheless we will derive it along with the general fact expressed in equation~\ref{GeiringerThmMainEq} to illustrate the newly enhanced methodology based on the lumping quotients of Markov chains described in subsection~\ref{LumpQuotSubsect}.}

%An important comment is in order here: it is possible that the denominator of one of the fractions involved in the product is $0$. However, in such a case, the numerator is also $0$ and we adopt the convention (in this theorem only) that if the numerator is $0$ then, regardless of the value of the denominator (i.e. even if the denominator is $0$), then the fraction is $0$. As a matter of fact, a denominator of some fraction involved is $0$ if and only if one of the following holds: $\alpha(P) = 0$ or if there exists an index $q$ with $1 \leq q \leq k-1$ such that no state in the similarity set $O_q$ appears in the population $P$ (and hence in either of the inflated populations $P_m$).
\end{thm}
\begin{rem}\label{simplRem}
Notice a simplification in the statement of theorem~\ref{GeiringerLikeThmForMCTMain} here: unlike the case in \cite{MitavRowGeirNewMain}, thanks to the assumption that the state space $S$ contains no states that are not present in the initial population $P$, the denominators of the multiples in the right hand side of equation~\ref{GeiringerThmMainEq} are never $0$. As mentioned before, this assumption does not reduce the generality since one can always shrink or enlarge the set of states $S$ and modify the collection $\mathcal{C}$ of similarity sets according to the population of rollouts under consideration.
\end{rem}
\begin{ex}\label{GeirThmForMCTMainEx}
Continuing with examples~\ref{setCoverEx}, \ref{popEx} and \ref{popStateOrderEx}, consider the Holland-Poli rollout schema \\ $h = (\beta, \, 4, \, 7, \, 5, \, f_2)$. Since the population $P$ consists of exactly $4$ rollouts, $b = 4$ and, since the action $\beta$ occurs only once in the population $P$, $\text{Numb}(\beta, \, P) = 1$. Then, according to theorem~\ref{GeiringerLikeThmForMCTMain}, we have $$\lim_{m \rightarrow \infty}\lim_{t \rightarrow \infty}\Phi(P_m, \, h, \, t) = \frac{1}{4} \cdot \frac{|4|}{|1 \cup 3 \cup 6|}\cdot \frac{|7|}{|7|} \frac{|5|}{|5|} \times$$
$$\times \frac{\emph{Order}(\beta \downarrow (1 \cup 3 \cup 6), \, P)}{\emph{Order}(\beta \downarrow \, (P))} \cdot \frac{\emph{Order}((1 \cup 3 \cup 6) \downarrow 7, \, P)}{\emph{Order}((1 \cup 3 \cup 6) \downarrow \, (P))}\times$$
$$\times \frac{\emph{Order}(7 \downarrow 5, \, P)}{\emph{Order}(7 \downarrow \, (P))}\cdot\frac{1}{\emph{Order}(5 \downarrow \, (P))} =$$$$\overset{\emph{after plugging in values from example~\ref{popStateOrderEx}}}{=}$$$$= \frac{1}{4} \cdot \frac{2}{7} \cdot 1 \cdot 1 \cdot 1 \cdot \frac{2}{7} \cdot \frac{1}{3} \cdot \frac{1}{3}=\frac{1}{441}.$$
\end{ex}
\subsection{Deriving Theorem~\ref{GeiringerLikeThmForMCTMain} for the MCT algorithm}\label{specificThmDerivationSubsect}
The crucial step in establishing theorem~\ref{GeiringerLikeThmForMCTMain} has been the observation that recombination operators can be defined over arbitrary set cover of the set of
state-action pairs in a bijective manner so that the general finite population Geiringer theorem in \cite{MitavRowGeirMain} (it is also presented and extended for non-homogenous time Markov processes in \cite{MitavRowGeirNewMain}) applies. This result is stated in theorem~\ref{GeirThmForMCTMain} of the current article. We now know that the stationary distribution of the Markov chain over the set of all populations started with an initial population $P_m$ that's obtained by inflating a given population $P$ by a factor of $m$, call it $[P_m]_{\mathcal{F}}$ (see subsection~\ref{specializationAndMainGeiringerThmSubSect} for the detailed description of the family of recombination transformations $\mathcal{F}$ and the associated Markov chain), is the uniform distribution over $[P_m]$. At the same time, we may consider the family of recombination transformations $\overline{\mathcal{F}}$ defined with respect to the set cover $\overline{\mathcal{C}}$ in place of $\mathcal{C}$ (the details of how this is done are presented in subsection~\ref{specializationAndMainGeiringerThmSubSect}).
Since $\overline{\mathcal{C}}$ is a partition of the set of states, we can now apply theorem 40 of \cite{MitavRowGeirNewMain} (the particular case of theorem~\ref{GeiringerLikeThmForMCTMain}) to deduce the following. Let $[P_m]_{\overline{\mathcal{F}}}$ denote the set of all populations that can be obtained from the population $P_{m}$ after applying the transformations from the family $\overline{\mathcal{F}}$. Then, if
\begin{equation}\label{schemaDefineProofEq}
\bar{h} = (\alpha, \, \overline{O_1}, \, \overline{O_2}, \ldots, \overline{O_{k-1}}, x_k)
\end{equation}
is a given Holland-Poli rollout schema, we have $$\lim_{m \rightarrow \infty}\lim_{t \rightarrow \infty}\Phi(P_m, \, \bar{h}, \, t) = $$
\begin{equation}\label{establishedThmEq}
= \frac{\text{Numb}(\alpha, \, P)}{b} \cdot \prod_{q = 1}^{k-1}\frac{\emph{Order}(\overline{O_{q-1}} \downarrow \overline{O_q}, \, P)}{\emph{Order}(\overline{O_{q-1}} \downarrow, P)} \cdot \emph{LF}(P, h)
\end{equation}
where $$\emph{LF}(P, h) = \begin{cases}
1 & \text{if } x_k = \#\\
0 & \text{if } x_k = f \in \Sigma \text{ and } f \notin \overline{O_{k-1}} \downarrow (P)\\
%\; & \\ %\; & \text{rollout in the population }P\\
\emph{Fraction} & \text{if } x_k = f \in \Sigma \text{ and } f \in \overline{O_{k-1}} \downarrow_{\Sigma} (P)
\end{cases}$$ and $$\emph{Fraction} = \frac{1}{\emph{Order}(\overline{O_{k-1}} \downarrow, P)}$$
Although there is a formal difference between the sets of populations $[P_m]_{\mathcal{F}}$ and $[P_m]_{\overline{\mathcal{F}}}$ mainly due to distinct labeling policies, if we extend the notion of a Holland-Poli rollout schema in a natural way to the set of all possible rollouts from the populations in $[P_m]_{\mathcal{F}}$ so that a rollout $r$ fits the schema $\bar{h}$ introduced in equation~\ref{schemaDefineProofEq} if and only if it fits at least one schema of the form $$h = (\alpha, \, T_1, \, T_2, \ldots, T_{k-1}, x_k)$$ where $\overline{T_i} = \overline{O_i}$ whenever $1 \leq i \leq k-1$ (In other wards, the set of rollouts defined by the schema $\bar{h}$ in equation~\ref{schemaDefineProofEq} is the union of the sets of rollouts represented by the schemata the extension of the similarity classes of which are precisely the corresponding equivalence classes appearing in $\bar{h}$), then we claim that $\forall \, m \in \mathbb{N}$, the fraction of occurrence of the schema $\bar{h}$ in the set of populations $[P_m]_{\mathcal{F}}$ is the same as it is in the populations $[P_m]_{\overline{\mathcal{F}}}$ (of course, since the corresponding unique stationary distributions of the two Markov chains are uniform, it follows that $\forall \, m \in \mathbb{N}$ $\lim_{t \rightarrow \infty}\Phi_{\mathcal{F}}(P_m, \, \bar{h}, \, t) = \lim_{t \rightarrow \infty}\Phi_{\overline{\mathcal{F}}}(P_m, \, \bar{h}, \, t)$). To see the assertion in the previous sentence, consider the functions $\phi_1: [P_m]_{\mathcal{F}} \rightarrow P_m^{\text{EquivSchemata}}$ and $\phi_2: [P_m]_{\overline{\mathcal{F}}} \rightarrow \overline{P}_m^{\text{EquivSchemata}}$ where the sets $P_m^{\text{EquivSchemata}}$ and $\overline{P}_m^{\text{EquivSchemata}}$ are obtained from the corresponding sets of populations in $[P_m]_{\mathcal{F}}$ and in $[P_m]_{\overline{\mathcal{F}}}$ respectively by replacing every label of the form $(O_1, \, O_2, \ldots, O_{l(\vec{u})} \, \vec{u})$ (where, of course, $O_i \in \{O \, | \, O \in \mathcal{C} \text{ and } \vec{u} \in O\}$) with the corresponding label $\overline{O_1}$ (notice that whenever $1 \leq i \leq l(\vec{u})$, we have $\overline{O_i} = \overline{O_1}$ since $\bigcap_{i=1}^{l(\vec{u})} O_i \supseteq \{\vec{u}\} \neq \emptyset$ so that all of the sets $O_i$ are within the same equivalence class of the transitive closure of the symmetric and reflexive relation induced by the set cover $\mathcal{C}$). The functions $\phi_1$ and $\phi_2$ are, essentially, the ``projections" onto the sets $P_m^{\text{EquivSchemata}}$ and $\overline{P}_m^{\text{EquivSchemata}}$: given a population $Q \in [P_m]_{\mathcal{F}}$ (or $Q \in [P_m]_{\overline{\mathcal{F}}}$), simply replace the labels of every rollout with the unique equivalence class label as described in the preceding sentence, thereby obtaining an element of $P_m^{\text{EquivSchemata}}$ (or $\overline{P}_m^{\text{EquivSchemata}}$). Notice that if we are given an initial population $P$ of rollouts, the states can be labeled either according to the similarity relation $\mathcal{C}$ or $\overline{\mathcal{C}}$ on the set of states. While these populations are formally distinct (in terms of labeling only: one of them is considered to be an initial population in the set $[P_m]_{\mathcal{F}}$ while the other one is the corresponding initial population in the set $[P_m]_{\overline{\mathcal{F}}}$), let's call them $P$ and $\overline{P}$, the corresponding images of these populations under the maps $\phi_1$ and $\phi_2$ are obviously identical. In fact, a lot more is true:
\begin{lem}\label{fact1LemProof}
$\forall \, m \in \mathbb{N}$ we have $$P_m^{\text{EquivSchemata}} = \overline{P}_m^{\text{EquivSchemata}}$$.
\end{lem}
\begin{proof}
Since there are more possibilities for recombination when using the family of transformations $\overline{\mathcal{F}}$ rather than the family of transformations $\mathcal{F}$ that only allows immediate swaps of sub-rollouts at the states that are subsets of a particular similarity set $O \subseteq \overline{O}$, it is clear that $$P_m^{\text{EquivSchemata}} \subseteq \overline{P}_m^{\text{EquivSchemata}}.$$ The reverse inclusion follows from the definition of the transitive closure of a symmetric and reflexive relation. In fact, as mentioned in the sentence preceding the statement of lemma~\ref{fact1LemProof}, $$P_m^{\text{EquivSchemata}} \cap \overline{P}_m^{\text{EquivSchemata}} \supseteq \{\phi_1(P_m) = \phi_2(\overline{P}_m)\} \neq \emptyset.$$ Since every population $U \in [P_m]_{\mathcal{F}}$ is obtained via a finite number of applications of the transformations of the form $\Theta_{\overline{O}, \, \vec{u}, \, \vec{v}} \in \overline{F}$ (see definition \ref{RecombStagePopTransDefn}), by the principle of induction, it is sufficient to show that if we are given a population $Q \in [P_m]_{\mathcal{F}}$ with $\phi_1(Q) = \phi_2(\overline{Q})$, then $\forall$ transformation of the form $\Theta_{\overline{O}, \, \vec{u}, \, \vec{v}} \in \overline{F}$, $\exists$ a finite sequence $\{\Theta_{O_j, \, \vec{u}_j, \, \vec{v}_j}\}_{j=1}^k$ of crossover transformations in $\mathcal{F}$ such that $\phi_2\left(\Theta_{\overline{O}, \, \vec{u}, \, \vec{v}}(\overline{Q})\right) =$
\begin{equation}\label{finalLemmaGoalEq}
\phi_1\left(\Theta_{O_k, \, \vec{u}_k, \, \vec{v}_k} \circ \Theta_{O_{k-1}, \, \vec{u}_{k-1}, \, \vec{v}_{k-1}} \circ \ldots \circ \Theta_{O_1, \, \vec{u}_1, \, \vec{v}_1}(Q) \right).
\end{equation}
Since $\vec{u} \simeq \vec{v}$ (see definition~\ref{transClosureDefn} and remark~\ref{equivRelRem}), $\exists \,$ a sequence of similarity classes $O_1, \, O_2, \ldots, O_{l-1}, \, O_l \in \mathcal{O}$ with $\overline{O_1} = \overline{O}$ and a corresponding sequence of states $\vec{u} = \vec{u}_1, \, \vec{u}_2, \ldots, \vec{u}_l = \vec{v}$ with $\vec{u}_1 \in O_1$, and, whenever $1 < i \leq l$, $\vec{u}_i \in O_{i-1} \cap O_i$ (evidently, in this case, $\forall \, i$ with $1 \leq i \leq l$, we have $\overline{O_i} = \overline{O_1} = \overline{O}$). Observing that the function $\phi_1$ (as well as $\phi_2$, of course), is invariant under the applications of the single swap crossover transformations (recall definitions~\ref{rolloutPartCrossDefn} and \ref{recombActOnPopsDef}) we deduce that $\phi_1(Q) = \phi_1(Q_1)$ where $$Q_1 = \nu_{O_{l-1}, \, \vec{u}_l, \, \vec{u}_{l-1}} \circ \nu_{O_3, \, \vec{u}_4, \, \vec{u}_3} \circ \ldots \circ \nu_{O_2, \, \vec{u}_3, \, \vec{u}_2}(Q).$$ At the same time, the corresponding states $\vec{u}$ and $\vec{v}$ in the population $Q_1$ are $O_1$ similar (recall the beginning of section~\ref{MathFrameworkSect}) so that $\phi_1\left(\Theta(O_1, \, \vec{u}, \, \vec{v})\right)(Q_1) = \phi_2\left(\Theta_{\overline{O}, \, \vec{u}, \, \vec{v}}(\overline{Q})\right)$ thereby producing a desired sequence of crossover transformations for the equality in \ref{finalLemmaGoalEq} to hold and finishing the argument.
\end{proof}
Evidently, the number of rollouts fitting the schema $\bar{h}$ in a population $Q \in [P_m]_{\mathcal{F}}$ is the same as that in the population $\phi_1(Q) \in P_m^{\text{EquivSchemata}}$ (in fact, this is precisely the way to count them, according to the way $\phi_1$ is defined). Likewise, of course, the same holds for the population $\overline{Q} \in [\overline{P}_m]_{\overline{F}}$: the number of rollouts fitting the schema $\bar{h}$ in the population $\overline{Q}$ is the same as the corresponding number in $\phi_2(\overline{Q}) \in \overline{P}_m^{\text{EquivSchemata}} = P_m^{\text{EquivSchemata}}$ according to lemma~\ref{fact1LemProof}. Evidently, the total number of states fitting a given similarity class $O \in \mathcal{C}$ (as well as these fitting its expansion, $\overline{O}$) remains invariant after an application of any of the possible recombination transformations. Since the functions $\phi_1$ and $\phi_2$ are both invariant under the applications of the single-swap crossover transformations of the form $\nu_{O, \, \vec{u}, \, \vec{u}}$, and, of course, every permutation is a composition of transpositions, for every population $H \in P_m^{\text{EquivSchemata}} = \overline{P}_m^{\text{EquivSchemata}}$ the size of the pre-image $\phi_1^{-1}(H) = \prod_{O \in \mathcal{C}} \left(m \cdot |O|\right)!$ while $\phi_2^{-1}(H) = \prod_{\overline{O} \in \overline{\mathcal{C}}} \left(m \cdot |\overline{O}|\right)!$, it follows, in particular, that $\exists \, k$ depending only on the inflation factor $m$, such that $\forall \, H \in P_m^{\text{EquivSchemata}} = \overline{P}_m^{\text{EquivSchemata}}$ we have $\phi_1^{-1}(H) = k \cdot \phi_2^{-1}(H)$. It is now apparent that the fraction of occurrence of rollouts fitting the schema $\bar{h}$ out of the total number of rollouts in $[P_m]$, is the same as that out of the total number of rollouts in $[\overline{P}_m]$ and is that out of $P_m^{\text{EquivSchemata}} = \overline{P}_m^{\text{EquivSchemata}}$. Thus, we have now shown the following intermediate fact:
\begin{lem}\label{prelimFrequencyOfOccurrenceLemma}
$\forall \, m \in \mathbb{N}$ we have $$\lim_{t \rightarrow \infty}\Phi_{\mathcal{F}}(P_m, \, \bar{h}, \, t) = \lim_{t \rightarrow \infty}\Phi_{\overline{\mathcal{F}}}(\overline{P}_m, \, \bar{h}, \, t)$$ so that, in particular, the equality in \ref{establishedThmEq} holds.
\end{lem}
The remaining part of the argument proceeds in a very similar manner as the proof of theorem 40 in \cite{MitavRowGeirNewMain}.\footnote{Certainly, lemma~\ref{prelimFrequencyOfOccurrenceLemma} can be established by nearly repeating the derivation of theorem 40 in \cite{MitavRowGeirNewMain}, yet the argument presented in the current paper is shorter.} Due to space limitations we provide only an outline of the argument reminding the cornerstones and the mathematical tools developed in \cite{MitavRowGeirMain} and largely enhanced in \cite{MitavRowGeirNewMain}. The first important step is the following fact (lemma 48 of \cite{MitavRowGeirNewMain}) that allows us to derive Geiringer-like results in terms of the fraction of populations where a rollout fitting a given schema $h$ occurs in a specified position (say, the first individual) in the population out of the total number of populations in $[P_m]_{\mathcal{F}}$.
\begin{lem}\label{mainGeiringerLikeLemma}
Given a subset $S \subseteq \Omega$ of rollouts and an initial population of rollouts, $P$, under the assumptions of theorem~\ref{GeiringerLikeThmForMCTMain}, it is true that
$$\lim_{t \rightarrow \infty}\Phi(S, \, P_m, \, t) = \frac{|\mathcal{V}(P_m, \, S)|}{|[P_m]_{\mathcal{F}}|}$$
where the set $\mathcal{V}(P_m, \, S)$ is the set of populations in $[P_m]_{\mathcal{F}}$ the first rollout of which, call it $r_1 \in S$.
\end{lem}
In view of lemma~\ref{mainGeiringerLikeLemma}, our goal is to estimate the ratio of the form $\frac{|\mathcal{V}(P_m, \, h)|}{|[P_m]_{\mathcal{F}}|}$ and, afterwards, to compute the limit as $m \rightarrow \infty$. We accomplish this task step by step: combining lemmas~\ref{prelimFrequencyOfOccurrenceLemma} and \ref{mainGeiringerLikeLemma} we deduce that $$\lim_{m \rightarrow \infty}\frac{|\mathcal{V}(P_m, \, \bar{h})|}{|[P_m]_{\mathcal{F}}|} = $$$$= \frac{\text{Numb}(\alpha, \, P)}{b} \cdot \prod_{q = 1}^{k-1}\frac{\emph{Order}(\overline{O_{q-1}} \downarrow \overline{O_q}, \, P)}{\emph{Order}(\overline{O_{q-1}} \downarrow, P)} \cdot \emph{LF}(P, h).$$ Certainly we can write $$\bar{h} = h_0 \supseteq h_1 \supseteq \ldots \supseteq h_{k-1} = h$$ where $$h_i = (\alpha, \, O_1, \, O_2, \ldots, O_{i-1}, \, O_i, \, \overline{O_{i+1}}, \, \overline{O_{i+2}}, \ldots, \overline{O_{k-1}}, x_k)$$ is the Holland-Poli schema with respect to the set cover $\mathcal{C} \cup \overline{\mathcal{C}}$. Now we can write $$\frac{|\mathcal{V}(P_m, \, h)|}{|[P_m]_{\mathcal{F}}|} = \frac{|\mathcal{V}(P_m, \, h_{k-1})|}{|\mathcal{V}(P_m, \, h_{k-2})|} \cdot \frac{|\mathcal{V}(P_m, \, h_{k-2})|}{|\mathcal{V}(P_m, \, h_{k-3})|} \cdot \ldots \cdot \frac{|\mathcal{V}(P_m, \, h_0)|}{|[P_m]_{\mathcal{F}}|}$$ as a ``telescoping" product so that
$$\lim_{m \rightarrow \infty}\frac{|\mathcal{V}(P_m, \, h)|}{|[P_m]_{\mathcal{F}}|} =$$$$=\left(\prod_{i=1}^{k-2}\lim_{m \rightarrow \infty}\frac{|\mathcal{V}(P_m, \, h_{k-i})|}{|\mathcal{V}(P_m, \, h_{k-i+1})|}\right) \cdot \lim_{m \rightarrow \infty}\frac{|\mathcal{V}(P_m, \, \bar{h})|}{|[P_m]_{\mathcal{F}}|}=$$$$=\left(\prod_{i=1}^{k-2}\lim_{m \rightarrow \infty}\frac{|\mathcal{V}(P_m, \, h_{k-i})|}{|\mathcal{V}(P_m, \, h_{k-i+1})|}\right) \times$$
\begin{equation}\label{decomposeComputeEqProof}
\times \frac{\text{Numb}(\alpha, \, P)}{b} \cdot \prod_{q = 1}^{k-1}\frac{\emph{Order}(\overline{O_{q-1}} \downarrow \overline{O_q}, \, P)}{\emph{Order}(\overline{O_{q-1}} \downarrow, P)} \cdot \emph{LF}(P, h).
\end{equation}
and, thanks to equation~\ref{decomposeComputeEqProof}, all that remains to establish theorem~\ref{GeiringerLikeThmForMCTMain} at this point is to show that $\forall \, i \in \{1, \, 2, \ldots, k-1\}$ we have
\begin{equation}\label{desiredEqGoalRatios}
\lim_{m \rightarrow \infty}\frac{|\mathcal{V}(P_m, \, h_i)|}{|\mathcal{V}(P_m, \, h_{i-1})|} = \frac{|O_i|}{|\overline{O_{i-1}}|}.
\end{equation}
The main tools involved in deriving equation~\ref{desiredEqGoalRatios}, just as in establishing theorem 40 in \cite{MitavRowGeirNewMain}, are the Markov inequality and the lumping quotients of Markov chains technique in the same way as in \cite{MitavRowGeirNewMain}. The lumping quotient method modified for specific applications such as in the current paper, is described in details in subsection 5.3 of \cite{MitavRowGeirNewMain}. Rather than estimating the ratio
\begin{equation}\label{goalRatioEqNoLimit}
R^m_i = \frac{|\mathcal{V}(P_m, \, h_i)|}{|\mathcal{V}(P_m, \, h_{i-1})|}
\end{equation}
directly, it is more convenient to estimate the closely related ratio
\begin{equation}\label{alternGoalRatioEqs}
\widetilde{R^m_i} = \frac{\mathcal{V}(P_m, \, h_i)}{\mathcal{V}(P_m, \, h_{i-1} \setminus h_i)}.
\end{equation}
Indeed, an elementary algebraic manipulation shows that
\begin{equation}\label{ElemAgebraicMnipEq}
R^m_i = \frac{1}{1+\frac{1}{\widetilde{R^m_i}}}.
\end{equation}
The way to estimate the ratio in equation~\ref{alternGoalRatioEqs}, is to construct a Markov chain (possibly non-irreducible) having a symmetric Markov transition matrix, so that the uniform probability distribution is one of its stationary distributions, call it $\pi_m^i$ on $\mathcal{V}(P_m, \, h_{i-1})$ and to express the ratio $\widetilde{R^m_i}$ in equation~\ref{alternGoalRatioEqs} in terms of the ratio of the corresponding probabilities under the uniform probability distribution $\pi_m^i$: $$\widetilde{R^m_i} = \frac{\pi_m^i(\mathcal{V}(P_m, \, h_i))}{\pi_m^i(\mathcal{V}(P_m, \, h_{i-1} \setminus h_i))}.$$ We then make use of lemma 55 in subsection 5.3 of \cite{MitavRowGeirNewMain} to estimate the ratio $\widetilde{R^m_i}$ in terms of the ratios of the corresponding generalized transition probabilities between the subsets
$$\frac{(1 - \delta)\lambda_1}{(1 - \delta)\kappa_2 + \delta} \leq \widetilde{R^m_i} = \frac{\pi_m^i\left(\mathcal{V}(P_m, \, h_i)\right)}{\pi_m^i\left(\mathcal{V}(P_m, \, h_{i-1} \setminus h_i)\right)} \leq$$
\begin{equation}\label{mainEstInequal}
\frac{(1 - \delta)\kappa_1 + \delta}{(1 - \delta)\lambda_2}
\end{equation}
where $\lambda_1$, $\lambda_2$, $\kappa_1$ and $\kappa_2$ are the corresponding bounds on the appropriate generalized transition probabilities (see subsection 5.3 of \cite{MitavRowGeirNewMain} for details), while $0 < \delta \ll 1$ is an arbitrary given small constant and the inequality \ref{mainEstInequal} holds $\forall \, m > M_{\delta}$. For the sake of completeness, lemma 55 of \cite{MitavRowGeirNewMain} is stated below.
\begin{lem}\label{estimStationaryRatiosLem}
Let $\{p_{x \rightarrow y}\}_{x, \, y \in \mathcal{X}}$ denote a Markov transition matrix over a finite state space $\mathcal{X}$. Suppose $\pi$ is a stationary distribution of this Markov chain (i.e. a fixed point of the associated Markov transition matrix) Suppose $A$ and $B \subseteq \mathcal{X}$ is a complementary pair of subsets (i.e. $A \cap B = \emptyset$ and $A \cup B = \mathcal{X}$). Suppose further that $U \subseteq \mathcal{X}$ is such that $$\frac{\pi(U \cap A)}{\pi(A)} < \delta < 1 \text{ and }\frac{\pi(U \cap B)}{\pi(B)} < \delta < 1.$$ Assume now that we find constants $\lambda_1$, $\lambda_2$, $\kappa_1$ and $\kappa_2$ such that $\forall \, b \in U^c \cap B$ we have $\lambda_1 \leq p_{b \rightarrow A} \leq \kappa_1$ and $\forall \, a \in U^c \cap A$ we have $\lambda_2 \leq p_{a \rightarrow B} \leq \kappa_2$. (given a state $x \in \mathcal{X}$ and a subset $Y \in \mathcal{X}$, $p_{x \rightarrow y} = \sum_{y \in Y}p_{x \rightarrow y}$). Then we have $$\frac{(1 - \delta)\lambda_1}{(1 - \delta)\kappa_2 + \delta} \leq \frac{\pi(A)}{\pi(B)} \leq \frac{(1 - \delta)\kappa_1 + \delta}{(1 - \delta)\lambda_2}$$
\end{lem}
Since the argument proceeds along exactly the same steps and ideas as the proof of theorem 40 presented in section 5 of \cite{MitavRowGeirNewMain}, we provide only a very rough and brief outline within a few sentences. The Markov chain on the state space $\mathcal{V}(P_m, \, h_{i-1})$ is constructed as follows: given a population $Q \in \mathcal{V}(P_m, \, h_{i-1})$, the first rollout of $P$ fits the schema $h_{i-1}$ (so that, in particular, the $i^{\text{th}}$ state of this rollout, call it $\vec{s} \in \overline{O_i}$). Let $\text{Mobile}(Q, \, i-1)$ denote the set of all states in $\overline{O_i}$ that do not appear within the first rollout in the population $Q$. Now select a state $\vec{t} \in \overline{O_i}$ uniformly at random and apply the one point crossover transformation $\chi_{\overline{O_i}, \vec{s}, \vec{t}}$ to the population $Q$, thereby obtaining a new population $\chi_{\overline{O_i}, \vec{s}, \vec{t}}(Q) \in \mathcal{V}(P_m, \, h_{i-1})$ that is different from $Q$ if and only if $\vec{s} \in \text{Mobile}(Q, \, i-1)$. Due to the fact that recombination is non-homologous (i.e. it may take place at distinct hight of various rollouts) the height of the first rollout may be arbitrarily large and, at first glance, it seems that the set $\text{Mobile}(Q, \, i-1)$ may vary in size greatly from population to population, nonetheless, we make a crucial observation that as the inflation factor $m \rightarrow \infty$, the sizes of each of the similarity sets $O \in \mathcal{C}$ and $\overline{O} \in \overline{\mathcal{C}}$ increase linearly by the factor of $m$ in size, while the average height of the first rollout in a population stays constant and is the same as the average height of the population $P$ (this is due to the fact that all recombination transformations preserve the total number of states within the population, and in particular, the average height of the population, and, at the same time, inflating a population by any factor $m \in \mathbb{N}$ also preserves the average height). A lovely application of the classical Markov inequality now shows that as the inflation factor $m \rightarrow \infty$, the probability that the first rollout contains a fixed ($\delta$-size) fraction of states from any given similarity set $O \in \mathcal{C}$ (and, even more so, $\overline{O} \in \overline{C}$) goes to $0$ (all of the technical details are entirely analogous to the ones presented in section 5 of \cite{MitavRowGeirNewMain}). Now let $A = \mathcal{V}(P_m, \, h_i)$, $B = \mathcal{V}(P_m, \, h_{i-1} \setminus h_i)$ and, for a given $\delta > 0$ select $M \in \mathbb{N}$ large enough so that $\forall \, m > M$
\begin{equation}\label{MarkovInequalConsequenceEq1}
\frac{\pi_m^i\left(\mathcal{V}(P_m, \, h_i) \cap U_m^{\delta \cdot const(i)}\right)}{\pi_m^i(\mathcal{V}(P_m, \, h_{i-1}))} < \delta
\end{equation}
and
\begin{equation}\label{MarkovInequalConsequenceEq2}
\frac{\pi_m^i\left(\mathcal{V}(P_m, \, h_{i-1} \setminus h_i) \cap U_m^{\delta \cdot const(i)}\right)}{\pi_m^i(\mathcal{V}(P_m, \, h_{i-1}))} < \delta
\end{equation}
where $U_m^{\delta \cdot const(i)}$ denotes the set of all populations in $[P_m]_{\mathcal{F}}$ with the property that the height of the first rollout in such populations is at least $M = \frac{E(H_1)}{(const(i) \cdot \delta)^2}$ and $H_1$ denotes the random variable measuring the height of the first rollout in a population selected from the set $[P_m]_{\mathcal{F}}$ uniformly at random, while $const(i) > 0$ depends only on the schema $h$ and the initial population $P$ of rollouts.\footnote{We invite the reader to study section 5 of \cite{MitavRowGeirNewMain} to understand how the constants $const(i) > 0$ are selected and why the inequalities~\ref{MarkovInequalConsequenceEq1} and \ref{MarkovInequalConsequenceEq2} hold.} It is easy to see from the construction of this auxiliary Markov chain, that the transition matrix $\{p_{Q \rightarrow R}\}_{Q \text{ and }R \in \mathcal{V}(P_m, \, h_{i-1})}$ is symmetric so that the uniform probability distribution $\pi_m^i$ is a stationary distribution of this Markov chain, and, furthermore, $\forall$ population $Q \in \mathcal{V}(P_m, \, h_{i-1} \setminus h_i) \setminus U_m^{\delta \cdot const(i)}$ we have $$\lambda_1 = \frac{|O_i| - \delta}{|\overline{O_i}|}=$$$$=\frac{m \cdot |O_i| - m \cdot \delta}{m \cdot |\overline{O_i}|} \leq \frac{m \cdot |O_i| - \emph{Mobile}(Q, \, i-1)}{m \cdot |\overline{O_i}|}=$$
\begin{equation}\label{sandwichBoundEq1}
p_{Q \rightarrow \mathcal{V}(P_m, \, h_i)} \leq \frac{m \cdot |O_i|}{m \cdot |\overline{O_i}|} = \frac{|O_i|}{|\overline{O_i}|} = \kappa_1
\end{equation}
and, likewise,  $\forall$ population $Q \in \mathcal{V}(P_m, \, h_i) \setminus U_m^{\delta \cdot const(i)}$ we have $$\lambda_2 = \frac{|\overline{O_i}| - |O_i| - \delta}{|\overline{O_i}|} =$$$$= \frac{m \cdot (|\overline{O_i}| - |O_i|) - m \cdot \delta}{m \cdot |\overline{O_i}|} = \frac{m \cdot (|\overline{O_i} \setminus O_i|) - m \cdot \delta}{m \cdot |\overline{O_i}|} \leq $$$$\leq \frac{m \cdot (|\overline{O_i}| - |O_i|) - \emph{Mobile}(Q, \, i-1)}{m \cdot |\overline{O_i}|}= p_{Q \rightarrow \mathcal{V}(P_m, \, h_i)} \leq$$
\begin{equation}\label{sandwichBoundEq2}
\leq \frac{m \cdot (|\overline{O_i}| - |O_i|)}{m \cdot |\overline{O_i}|} = \frac{|\overline{O_i}| - |O_i|}{|\overline{O_i}|} = \kappa_2
\end{equation}
so that lemma~\ref{estimStationaryRatiosLem} immediately yields the bounds $$\frac{\frac{|O_i| - \delta}{|\overline{O_i}|}(1-\delta)}{(1-\delta)\frac{|\overline{O_i}| - |O_i|}{|\overline{O_i}|}+\delta} \leq \frac{\pi_m^i\left(\mathcal{V}(P_m, \, h_i)\right)}{\pi_m^i\left(\mathcal{V}(P_m, \, h_{i-1} \setminus h_i)\right)} \leq$$
\begin{equation}\label{finalProofBoundIneq}
\frac{(1 - \delta)\frac{|O_i|}{|\overline{O_i}|} + \delta}{(1 - \delta)\frac{|\overline{O_i}| - |O_i| - \delta}{|\overline{O_i}|}}
\end{equation}
Given any $\delta > 0$, the ``sandwich" bounds in the inequality~\ref{finalProofBoundIneq} hold for all sufficiently large $m$ depending on the $\delta$ so that taking the limit of both sides as $\delta \rightarrow 0$ finally tells us that $$\lim_{m \rightarrow \infty}\frac{|\mathcal{V}(P_m, \, h_i)|}{|\mathcal{V}(P_m, \, h_{i-1} \setminus h_i)|}=$$
\begin{equation}\label{finalLimitEqProof}
= \lim_{m \rightarrow \infty}\frac{\pi_m^i\left(\mathcal{V}(P_m, \, h_i)\right)}{\pi_m^i\left(\mathcal{V}(P_m, \, h_{i-1} \setminus h_i)\right)}=\frac{|O_i|}{|\overline{O_i}|-|O_i|}
\end{equation}
and equation~\ref{desiredEqGoalRatios} follows from equation~\ref{finalLimitEqProof} via equation~\ref{ElemAgebraicMnipEq} thereby finishing the proof of theorem~\ref{GeiringerLikeThmForMCTMain}.
\section{Conclusions}
In the current paper we have significantly generalized a novel version of a finite population Geiringer-like theorem with non-homologous recombination established in \cite{MitavRowGeirNewMain} by allowing similarity relations on the set of state-action pairs to be modeled in terms of arbitrary set covers of the state-action space (not necessarily partitions induced by equivalence relations). This raises questions regarding further potential applications of Geiringer-like theorems for decision making to the design, of novel algorithms for pay-off-based clustering.

%\section{Acknowledgments}

%\bibliographystyle{unsrt}
%\bibliography{boris}  % sigproc.bib is the name of the Bibliography in this case
% ACM needs 'a single self-contained file'!
%\balancecolumns
\end{document}